\documentclass[10pt,final,journal,twocolumn]{IEEEtran}

\usepackage{amsmath,amssymb,color,amsmath, graphicx, float, caption, subcaption, mathrsfs,color} \usepackage[ruled]{algorithm2e} 
\usepackage{bm}
\usepackage{epstopdf }
\usepackage[font={small}]{caption}
\usepackage{cite}
\usepackage{amsthm}
\usepackage{url}

\theoremstyle{plain}
\newtheorem{theorem}{\textbf{Theorem}}

\newtheorem{lemma}[theorem]{\textbf{Lemma}}

\newtheorem{proposition}[theorem]{\textbf{Proposition}}
\newtheorem{condition}{\textbf{Condition}}

\def\H{\bm{H}}
\def\I{\bm{I}}

\def\P{\bm{P}}
\def\Q{\bm{Q}}

\def\e{\bm{e}}

\def\g{\bm{g}}
\def\h{\bm{h}}

\def\u{\bm{u}}
\def\v{\bm{v}}

\def\x{\bm{x}}
\def\y{\bm{y}}
\def\z{\bm{z}}

\def\argmin#1{\underset{#1}{\textrm{argmin}}}

\def\minim#1{\underset{#1}{\textrm{min}}}

\def\0{\mathbf{0}}
\def\1{\mathbf{1}}

\newcommand{\putSpace}[1]{\textcolor{white}{#1}}

\SetKwInput{kwEvaluate}{Evaluate}
\SetKwInput{kwSort}{Sort}
\SetKwInput{kwInput}{Input}
\SetKwInput{kwOutput}{Output}
\SetKwInput{kwInitialize}{Initialize}
\SetKwInput{kwParameters}{Params.}
\SetKwInput{kwDefaults}{Default init.}

\begin{document}

\title{
Image Restoration by Iterative Denoising and Backward Projections \\ }
\author{
    Tom~Tirer,
    and Raja~Giryes \\

    \thanks{
(c) 2018 IEEE. Personal use of this material is permitted. However, permission to use this material for any other purposes must be obtained from the IEEE by sending a request to pubs-permissions@ieee.org.

The authors are with the School of Electrical Engineering, Tel Aviv University, Tel Aviv 69978, Israel. (email: tirer.tom@gmail.com, raja@tauex.tau.ac.il)}
          }
\markboth{To appear in IEEE Transactions on Image Processing}{Tirer at al.
          : Performance Analysis of High Resolution Direct Position Determination Method}
\maketitle

\begin{abstract}

Inverse problems appear in many applications, such as image deblurring and inpainting. The common approach to address them is to design a specific algorithm for each problem. The Plug-and-Play (P\&P) framework, which has been recently introduced, allows solving general inverse problems by leveraging the impressive capabilities of existing denoising algorithms. While this fresh strategy has found many applications, a burdensome parameter tuning is often required in order to obtain high-quality results. In this work, we propose an alternative method for solving inverse problems using off-the-shelf denoisers, which requires less parameter tuning. 
First, we transform a typical cost function, composed of fidelity and prior terms, into a closely related, novel optimization problem.
Then, we propose an efficient minimization scheme with a plug-and-play property, i.e., the prior term is handled solely by a denoising operation.
Finally, we present an automatic tuning mechanism to set the method's parameters. 
We provide a theoretical analysis of the method, and empirically demonstrate its competitiveness with task-specific techniques and 
the P\&P approach for image inpainting and deblurring.

\end{abstract}


\section{Introduction}

We consider the reconstruction of an image from its degraded version, which may be noisy, blurred, downsampled, or all together. This general problem has many important applications, such as medical imaging, surveillance, entertainment, and more.
Traditionally, the design of task-specific algorithms has been the ruling approach. Many works specifically considered image denoising \cite{dabov2007image, elad2006image, buades2005review}, deblurring \cite{delbracio2015burst, guerrero2008image, danielyan2012bm3d}, inpainting \cite{bertalmio2000image, criminisi2004region, elad2005simultaneous}, super-resolution \cite{kim2016accurate, romano2017raisr}, etc. 

Recently, a new approach attracts much interest. This approach suggests leveraging the impressive capabilities of existing denoising algorithms for solving other tasks that can be formulated as an inverse problem. The concept is introduced in the Plug-and-Play (P\&P) method \cite{venkatakrishnan2013plug}, which presents an elegant way to decouple the measurement model and the image prior, such that the latter is handled solely by a denoising operation. Thus, it is not required to explicitly specify the prior, since it is implicitly defined through the choice of the denoiser. We note that several earlier works \cite{beck2009fast, afonso2010fast} have solved linear inverse problems iteratively, where a denoising sub-problem is solved in each iteration. Yet, these methods assume an explicit prior term, while \cite{venkatakrishnan2013plug} demonstrates the advantages of using well-known denoisers, even when it is not clear how to formulate their associated priors.

The P\&P method has already found many applications, e.g. bright field electron tomography \cite{sreehari2016plug}, Poisson
denoising \cite{rond2016poisson}, and postprocessing of compressed images \cite{dar2016postprocessing}. It also inspired new related techniques \cite{meinhardt2017learning, liu2017nonlocal, romano2016little, teodoro2016image, zhang2017learning, bigdeli2017image}. However, it has been noticed that the P\&P often requires a burdensome parameter tuning in order to obtain high quality results \cite{romano2016little, chan2017plug}. Moreover, since it is an iterative method, sometimes a large number of iterations (consisting of denoising operations) is required.

In this work, we propose a simple iterative method for solving linear inverse problems using denoising algorithms, which provides an alternative to P\&P. 
Our strategy requires less parameter tuning, and often less iterations than P\&P. Its recovery performance is competitive with task-specific algorithms and with the P\&P approach. 
To derive our algorithm, we first transform a typical cost function, composed of fidelity and prior terms, into a closely related, novel optimization problem.
Then, we propose an efficient minimization scheme with the desired plug-and-play property for the prior term. 
Finally, we provide an automatic tuning mechanism to set the method's parameters. 
We demonstrate the advantages of the new technique on inpainting and deblurring problems. 

Perhaps the most appealing property of the proposed strategy is its minimal parameter tuning. Specifically, for the noisy inpainting problem, our method has a single parameter that can be just set to zero, and for the deblurring problem we suggest an automatic parameter tuning scheme that can be employed. Regarding the latter, we note that there are other works that consider automatic parameter selection in inverse problems. However, in these works the prior term is restricted to certain types of penalty functions, e.g. Tikhonov regularization \cite{golub1979generalized, hansen1993use, haber2000gcv}, smoothed versions of the $\ell_p$ ($1<p<2$) norm \cite{eldar2009generalized, giryes2011projected}, or even more general convex functions \cite{ramani2012iterative, ramani2012regularization}. As far as we know, the literature does not offer similar tuning mechanism for sophisticated non-convex priors (e.g. BM3D \cite{dabov2007image}), all the more so for learned priors (e.g. IRCNN \cite{zhang2017learning}). In contrast, the tuning considerations of our method do not depend on the prior, which is arbitrarily specified by the chosen denoiser.

The paper is organized as follows. In Section \ref{sec_background} we present the problem formulation and the P\&P approach.
In Section \ref{sec_IDBP_algorithm} we present the proposed algorithm, provide a practical way to tune its parameter, and discuss its usage for inpainting and deblurring problems.
Section \ref{sec_analysis} includes mathematical analysis of the algorithm.
In Section \ref{sec_Experiments} the proposed method is empirically examined for the inpainting and deblurring problems. Section \ref{sec_conclusion} concludes the paper.

\section{Background}
\label{sec_background}

\subsection{Problem formulation}
\label{sec_problem_formulation}

The problem of image restoration can be generally formulated by
\begin{align}
\label{Eq_general_model}
\y = \H\x + \e,
\end{align}
where $\x \in \Bbb R^n$ represents the unknown original image, $\y \in \Bbb R^m$ represents the observations, $\H$ is an $m \times n$ degradation matrix and $\e \in \Bbb R^m$ is a vector of independent and identically distributed Gaussian random variables with zero mean and standard deviation of $\sigma_e$. The model in (\ref{Eq_general_model}) can represent different image restoration problems; for example: image denoising when $\H$ is the $n \times n$ identity matrix $\I_n$, image inpainting when $\H$ is a selection of $m$ rows of $\I_n$, and image deblurring when $\H$ is a blurring operator. 

In all of these cases, a prior image model $s(\x)$ is required in order to successfully estimate $\x$ from the observations $\y$.
Specifically, note that $\H$ is ill-conditioned in the case of image deblurring, thus, in practice it can be approximated by a rank-deficient matrix, or alternatively by a full rank $m \times n$ matrix ($m<n$). Therefore, for a unified formulation of inpainting and deblurring problems, which are the test cases of this paper, we assume $m<n$.

Almost any approach for recovering $\x$ involves formulating a cost function, composed of fidelity and penalty terms, which is minimized by the desired solution. The fidelity term ensures that the solution agrees with the measurements, and is often derived from the negative log-likelihood function. The penalty term regularizes the optimization problem through the prior image model $s(\x)$.
Hence, the typical cost function is
\begin{align}
\label{Eq_cost_func1}
f(\tilde{\x}) = \frac{1}{2\sigma_e^2} \| \y-\H\tilde{\x} \|_2^2 + s(\tilde{\x}),
\end{align}
where $\tilde{\x}$ is the optimization variable, and $\| \cdot \|_2$ stands for the Euclidean norm.

\subsection{Plug and Play approach}
\label{sec_PnP_approach}

Instead of devising a separate algorithm to solve $\minim{\tilde{\x}} f(\tilde{\x})$ for each type of matrix $\H$, a general recovery strategy has been proposed in \cite{venkatakrishnan2013plug}, denoted as the Plug-and-Play (P\&P). For completeness, we briefly describe this technique. 

Using variable splitting, the P\&P method restates the minimization problem as
\begin{align}
\label{Eq_pnp_cost_func2}
\minim{\tilde{\x}, \tilde{\v}} \,\,\,\, \ell(\tilde{\x}) + \beta s(\tilde{\v}) \,\,\,\, \textrm{s.t.} \,\,\,\, \tilde{\x}=\tilde{\v},
\end{align}
where $\ell(\tilde{\x}) \triangleq \frac{1}{2\sigma_e^2} \| \y-\H\tilde{\x} \|_2^2$ is the fidelity term in (\ref{Eq_cost_func1}), and $\beta$ is a positive parameter that adds flexibility to the cost function.
This problem can be solved using ADMM \cite{boyd2011distributed} by constructing an augmented Lagrangian, which is given by
\begin{align}
\label{Eq_pnp_AL}
L_\lambda &= \ell(\tilde{\x}) + \beta s(\tilde{\v}) + \u^T(\tilde{\x}-\tilde{\v}) + \frac{\lambda}{2} \| \tilde{\x}-\tilde{\v} \|_2^2 \nonumber \\
&=  \ell(\tilde{\x}) + \beta s(\tilde{\v}) + \frac{\lambda}{2} \| \tilde{\x}-\tilde{\v}+\tilde{\u} \|_2^2 - \frac{\lambda}{2} \|  \tilde{\u} \|_2^2,
\end{align}
where $\u$ is the dual variable, $\tilde{\u} \triangleq \frac{1}{\lambda}\u$ is the scaled dual variable, and $\lambda$ is the ADMM penalty parameter. The ADMM algorithm consists of iterating until convergence over the following three steps
\begin{align}
\label{Eq_pnp_ADMM}
\check{\x}_k &= \argmin{\tilde{\x}} \, L_\lambda(\tilde{\x},\check{\v}_{k-1},\check{\u}_{k-1}), \nonumber \\
\check{\v}_k &= \argmin{\tilde{\v}} \, L_\lambda(\check{\x}_{k},\tilde{\v},\check{\u}_{k-1}), \nonumber \\
\check{\u}_k &= \check{\u}_{k-1} + (\check{\x}_{k}  - \check{\v}_{k} ).
\end{align}
By plugging (\ref{Eq_pnp_AL}) in (\ref{Eq_pnp_ADMM}) we have
\begin{align}
\label{Eq_pnp_ADMM2}
\check{\x}_k &= \argmin{\tilde{\x}} \, \ell(\tilde{\x}) + \frac{\lambda}{2} \| \tilde{\x}-(\check{\v}_{k-1}-\check{\u}_{k-1}) \|_2^2, \nonumber \\
\check{\v}_k &= \argmin{\tilde{\v}} \, \frac{\lambda}{2\beta} \| (\check{\x}_k+\check{\u}_{k-1})-\tilde{\v} \|_2^2 + s(\tilde{\v}), \nonumber \\
\check{\u}_k &= \check{\u}_{k-1} + (\check{\x}_{k}  - \check{\v}_{k} ).
\end{align}
Note that the first step in (\ref{Eq_pnp_ADMM2}) is just solving a least squares (LS) problem and the third step is a simple update. The second step is more interesting. It describes obtaining $\check{\v}_k$ using a denoiser for white Gaussian noise of variance $\sigma^2=\beta/\lambda$, applied on the image $\check{\x}_k+\check{\u}_{k-1}$. 
This can be written compactly as $\check{\v}_k=\mathcal{D}(\check{\x}_k+\check{\u}_{k-1};\sigma)$, where $\mathcal{D}(\cdot;\sigma)$ is a denoising operator. Since general denoising algorithms can be used to implement the operator $\mathcal{D}(\cdot;\sigma)$,
the P\&P method does not require knowing or explicitly specifying the prior function $s(\x)$. Instead, $s(\x)$ is implicitly defined through the choice of $\mathcal{D}(\cdot;\sigma)$.
The obtained P\&P algorithm is presented in Algorithm \ref{PnP_alg}.

\begin{algorithm}
\caption{Plug and Play (P\&P)}
\vspace{2mm}
\kwInput{$\H, \y, \sigma_e$, denoising operator $\mathcal{D}(\cdot;\sigma)$, stopping criterion. $\y = \H\x+\e$, such that $\e \sim \mathcal{N}(\0,\sigma_e^2\I_m)$ and $\x$ is an unknown signal whose prior model is specified by $\mathcal{D}(\cdot;\sigma)$.}
\kwOutput{$\hat{\x}$ an estimate for $\x$.}
\kwInitialize{$\check{\v}_0=$ some initialization, $\check{\u}_0=\0$, $k=0$, some initialization for $\beta$ and $\lambda$.}
\While{stopping criterion not met}{
    $k = k+1$\;
    $\check{\x}_k = (\H^T\H+\lambda\sigma_e^2\I_n)^{-1} \times \hspace{80pt} \putSpace{\,\,\,\,\,\,\,\,\,\,\,\,\,\,\,\,\,\,\,\,\,\,\,\,\,\,\,} (\H^T\y + \lambda\sigma_e^2(\check{\v}_{k-1}-\check{\u}_{k-1}))$\;
    $\check{\v}_k = \mathcal{D}(\check{\x}_k+\check{\u}_{k-1};\sqrt{\beta/\lambda})$\;
    $\check{\u}_k = \check{\u}_{k-1} + (\check{\x}_{k}  - \check{\v}_{k})$\;
}
$\hat{\x} = \check{\x}_k$\;
\label{PnP_alg}
\end{algorithm}

From ADMM theory, {\em global convergence} (i.e. iterations approach feasibility and objective reaches its optimal value) is ensured if $\ell(\x)$ and $s(\x)$ are convex, closed, proper, and the unaugmented Lagrangian has a saddle point \cite{boyd2011distributed}. Yet, the immediate implication of this result for P\&P is limited, as the prior functions associated with popular off-the-shelf denoisers are non-convex or even unclear.
Avoiding the specification of $s(\x)$, global convergence of P\&P is proved in \cite{sreehari2016plug} for a denoiser $\mathcal{D}(\cdot;\sigma)$ that has a symmetric gradient and is non-expansive. However, the latter is difficult to be proved, and well-known denoisers such as BM3D \cite{dabov2007image}, K-SVD \cite{elad2006image}, and standard NLM \cite{buades2005review}, lead to good results despite violating these conditions.
Another type of convergence is {\em fixed point convergence}, which guarantees that an iterative algorithm asymptotically enters a steady state. A modified version of P\&P, where the ADMM parameter $\lambda$ increases between iterations, is guaranteed to have such a convergence under some mild conditions on the denoiser \cite{chan2017plug}.

The P\&P method is not free of drawbacks. Its main difficulties are the large number of iterations, which is often required by the P\&P to converge to a good solution, and the setting of the design parameters $\beta$ and $\lambda$, which is not always clear and strongly affects the performance.

\section{The Proposed Algorithm}
\label{sec_IDBP_algorithm}

In this work we take another strategy for solving inverse problems using denoising algorithms.
We start with formulating the cost function (\ref{Eq_cost_func1}) in somewhat strange but equivalent way
\begin{align}
\label{Eq_cost_func2}
f(\tilde{\x}) & = \frac{1}{2\sigma_e^2} \| \y-\H\tilde{\x} \|_2^2 + s(\tilde{\x}) \nonumber \\
&= \frac{1}{2\sigma_e^2} \| \H( \H^\dagger\y-\tilde{\x}) \|_2^2 + s(\tilde{\x})  \nonumber \\
&=\frac{1}{2\sigma_e^2} \| \H^\dagger\y-\tilde{\x} \|_{\H^T\H}^2 + s(\tilde{\x}),
\end{align}
where
\begin{align}
\label{Eq_cost_func2_def1}
\H^\dagger & \triangleq \H^T(\H\H^T)^{-1}  \\
\label{Eq_cost_func2_def2}
\| \u \|_{\H^T\H}^2 & \triangleq  \u^T\H^T\H\u.
\end{align}
Note that $\H^\dagger$ is the pseudoinverse of the full row rank matrix $\H$, and $\| \u \|_{\H^T\H}$ is a seminorm rather than a real norm, since $\H^T\H$ is not a positive definite matrix in our case. Moreover, as mentioned above, since the null space of $\H^T\H$ is nontrivial, the prior $s(\tilde{\x})$ is essential in order to obtain a meaningful solution.

The optimization problem $\minim{\tilde{\x}} f(\tilde{\x})$ can be equivalently written as
\begin{align}
\label{Eq_cost_func3}
\minim{\tilde{\x}, \tilde{\y}} \,\,\, \frac{1}{2\sigma_e^2} \| \tilde{\y}-\tilde{\x} \|_{\H^T\H}^2 + s(\tilde{\x}) \,\,\,\, \textrm{s.t.} \,\,\,\, \tilde{\y}=\H^\dagger \y. 
\end{align}
Note that due to the degenerate constraint, the solution for $\tilde{\y}$ is trivial $\tilde{\y}=\H^\dagger \y$.

Now, we make two major modifications to the above optimization problem. The basic idea is to loosen the variable $\tilde{\y}$ in a restricted manner, with the purpose of facilitating the estimation of $\x$.
First, we give some degrees of freedom to $\tilde{\y}$ by using the constraint $\H\tilde{\y}=\y$ instead of $\tilde{\y}=\H^\dagger \y$. 
Note, though, that components of $\tilde{\y}$ in the null space of $\H$ are ignored by the current fidelity term and the new constraint, because in both of them $\tilde{\y}$ is multiplied by $\H$.
Since these components are not controlled, they may strongly disagree with the prior $s(\tilde{\x})$ and complicate the optimization with respect to $\tilde{\x}$.
Therefore, to tackle this issue, we replace the seminorm $\frac{1}{\sigma_e^2} \| \tilde{\y}-\tilde{\x} \|_{\H^T\H}^2$ in the fidelity term with the Euclidean norm $\frac{1}{(\sigma_e+\delta)^2} \| \tilde{\y}-\tilde{\x} \|_2^2$, where $\delta$ is a design parameter. 
This leads to the following optimization problem
\begin{align}
\label{Eq_cost_func_our}
\minim{\tilde{\x}, \tilde{\y}} \,\,\, \frac{1}{2(\sigma_e+\delta)^2} \| \tilde{\y}-\tilde{\x} \|_2^2 + s(\tilde{\x}) \,\,\,\, \textrm{s.t.} \,\,\,\, \H\tilde{\y}= \y. 
\end{align}

Note that $\delta$ introduces a tradeoff. On the one hand, exaggerated value of $\delta$ should be avoided, as it may over-reduce the effect of the fidelity term. On the other hand, too small value of $(\sigma_e+\delta)^2$ may over-penalize $\tilde{\x}$ unless it is very close to the affine subspace $\{ \H \Bbb R^n = \y \}$. This limits the effective feasible set of $\tilde{\x}$ in problem (\ref{Eq_cost_func_our}), such that it may not include potential solutions of the original problem (\ref{Eq_cost_func3}).
Therefore, we suggest setting the value of $\delta$ as
\begin{align}
\label{Eq_delta_cond}
\delta &= \argmin{\tilde{\delta}} \,\, (\sigma_e+\tilde{\delta})^2 \nonumber \\
&\textrm{s.t.} \,\,\,\,
\frac{1}{\sigma_e^2} \| \H^\dagger\y-\tilde{\x} \|_{\H^T\H}^2 \geq \frac{1}{(\sigma_e+\tilde{\delta)}^2} \| \tilde{\y}-\tilde{\x} \|_2^2  \nonumber \\
&\forall \,\,\, \tilde{\x}, \tilde{\y} \in \mathcal{S}_{(\ref{Eq_cost_func_our})},
\end{align}
where $\mathcal{S}_{(\ref{Eq_cost_func_our})}$ denotes the feasible set of problem (\ref{Eq_cost_func_our}).
Note that the feasibility of $\tilde{\x}$ is dictated by $s(\tilde{\x})$\footnote{Since we make no assumptions on the prior function, it may define an arbitrary feasible set. For example, it can be the characteristic function of some set $\Omega$, i.e. $s(\tilde{\x})=
\begin{cases} 
      0, & \tilde{\x} \in \Omega \\
      +\infty,  & \tilde{\x} \notin \Omega
   \end{cases}
$.} and the feasibility of $\tilde{\y}$ is dictated by the constraint in (\ref{Eq_cost_func_our}).
The problem of obtaining such value for $\delta$ (or an approximation) is discussed in Section \ref{finding_delta}, where a relaxed version of the condition in (\ref{Eq_delta_cond}) is presented.

Assuming that $\delta$ solves (\ref{Eq_delta_cond}), the property that $\frac{1}{\sigma_e^2} \| \H^\dagger\y-\tilde{\x} \|_{\H^T\H}^2 \approx \frac{1}{(\sigma_e+\delta)^2} \| \tilde{\y}-\tilde{\x} \|_2^2$ for feasible $\tilde{\x}$ and $\tilde{\y}$, together with the fact that $\tilde{\y}=\H^\dagger \y$ is one of the solutions of the underdetermined system $\H\tilde{\y}= \y$, prevents increasing the penalty on potential solutions of the original optimization problem (\ref{Eq_cost_func3}). Therefore, roughly speaking, we do not lose solutions when we solve (\ref{Eq_cost_func_our}) instead of (\ref{Eq_cost_func3}).
As a sanity check, observe that if $\H=\I_n$ then the constraint in (\ref{Eq_cost_func_our}) degenerates to $\tilde{\y}=\y$ and the solution to (\ref{Eq_delta_cond}) is $\delta=0$. Therefore, (\ref{Eq_cost_func_our}) reduces to the original image denoising problem.

An additional insight on the new optimization problem is given in Appendix \ref{app:insight}, where we try to explain, from a numerical optimization point of view, why minimizing (\ref{Eq_cost_func_our}) rather than (\ref{Eq_cost_func3}) might even end up with a solution closer to the true image $\x$.

We solve (\ref{Eq_cost_func_our}) using alternating minimization. Iteratively, $\tilde{\x}_k$ is estimated by solving
\begin{align}
\label{Eq_cost_func_our_x}
\tilde{\x}_k = \argmin{\tilde{\x}} \,\, \frac{1}{2(\sigma_e+\delta)^2} \| \tilde{\y}_{k-1}-\tilde{\x} \|_2^2 + s(\tilde{\x}),
\end{align}
and $\tilde{\y}_k$ is estimated by solving
\begin{align}
\label{Eq_cost_func_our_y}
\tilde{\y}_k = \argmin{\tilde{\y}} \,\, \| \tilde{\y}-\tilde{\x}_k \|_2^2  \,\,\,\, \textrm{s.t.} \,\,\,\, \H\tilde{\y}= \y,
\end{align}
which describes a projection of $\tilde{\x}_k$ onto the affine subspace $\{ \H \Bbb R^n = \y \}$, and has a closed-form solution
\begin{align}
\label{Eq_cost_func_our_y_sol}
\tilde{\y}_k = \H^\dagger\y + (\I_n - \H^\dagger\H)\tilde{\x}_k.
\end{align}
Similarly to the P\&P technique, (\ref{Eq_cost_func_our_x}) describes obtaining $\tilde{\x}_k$ using a denoiser for white Gaussian noise of variance $\sigma^2=(\sigma_e+\delta)^2$, applied on the image $\tilde{\y}_{k-1}$,
and can be written compactly as $\tilde{\x}_k=\mathcal{D}(\tilde{\y}_{k-1};\sigma)$, where $\mathcal{D}(\cdot;\sigma)$ is a denoising operator. 
Moreover, as in the case of the P\&P,
the proposed method does not require knowing or explicitly specifying the prior function $s(\x)$. Instead, $s(\x)$ is implicitly defined through the choice of $\mathcal{D}(\cdot;\sigma)$.

The variable $\tilde{\y}_k$ is expected to be closer to the true signal $\x$ than the raw observations $\y$. Thus, our algorithm alternates between estimating the signal and using this estimation to obtain improved measurements (that also comply with the original observations $\y$).
The proposed algorithm, which we call Iterative Denoising and Backward Projections (IDBP), is presented in Algorithm \ref{IDBP_alg}.

\begin{algorithm}
\caption{Iterative Denoising and Backward Projections (IDBP)}
\vspace{2mm}
\kwInput{$\H, \y, \sigma_e$, denoising operator $\mathcal{D}(\cdot;\sigma)$, stopping criterion. $\y = \H\x+\e$, such that $\e \sim \mathcal{N}(\0,\sigma_e^2\I_m)$ and $\x$ is an unknown signal whose prior model is specified by $\mathcal{D}(\cdot;\sigma)$.}
\kwOutput{$\hat{\x}$ an estimate for $\x$.}
\kwInitialize{$\tilde{\y}_0=$ some initialization, $k=0$, $\delta$ approx. satisfying (\ref{Eq_delta_cond}).}
\While{stopping criterion not met}{
    $k = k+1$\;
    $\tilde{\x}_k=\mathcal{D}(\tilde{\y}_{k-1};\sigma_e+\delta)$\;
    $\tilde{\y}_k = \H^\dagger\y + (\I_n - \H^\dagger\H)\tilde{\x}_k$\;
}
$\hat{\x} = \tilde{\x}_k$\;
\label{IDBP_alg}
\end{algorithm}

\subsection{Setting the value of the parameter $\delta$}
\label{finding_delta}

Setting the value of $\delta$ that solves (\ref{Eq_delta_cond}) is required for simple theoretical justification of our method.
However, it is not clear how to obtain such $\delta$ in general.
Therefore, in order to relax the condition in (\ref{Eq_delta_cond}), that should be satisfied by all $\tilde{\x}$ and $\tilde{\y}$ in $\mathcal{S}_{(\ref{Eq_cost_func_our})}$, we can focus only on the sequences $\{\tilde{\x}_k\}$ and $\{\tilde{\y}_k\}$ generated by the proposed alternating minimization process. Then, we can use the following proposition.

\begin{proposition}
\label{proposition1}
Set $\delta=\tilde{\delta}$. If there exist an iteration $k$ of IDBP that violates the following condition
\begin{align}
\label{Eq_delta_cond3}
\frac{1}{\sigma_e^2} \| \y-\H\tilde{\x}_k \|_2^2 \geq \frac{1}{(\sigma_e+\tilde{\delta})^2} \| \H^\dagger ( \y - \H\tilde{\x}_k) \|_2^2,
\end{align}
then $\delta=\tilde{\delta}$ also violates the condition in (\ref{Eq_delta_cond}).
\end{proposition}

\begin{proof}
Assume that $\tilde{\x}_k$ generated by IDBP at some iteration $k$ violates (\ref{Eq_delta_cond3}), then it also violates the equivalent condition
\begin{align}
\label{Eq_delta_cond2}
\frac{1}{\sigma_e^2} \| \H^\dagger\y-\tilde{\x}_k \|_{\H^T\H}^2 \geq \frac{1}{(\sigma_e+\tilde{\delta})^2} \| \H^\dagger \y - \H^\dagger\H\tilde{\x}_k \|_2^2,
\end{align}
where we use $\| \H^\dagger\y-\tilde{\x}_k \|_{\H^T\H}^2 = \| \H(\H^\dagger\y-\tilde{\x}_k) \|_2^2 = \| \y- \H \tilde{\x}_k \|_2^2 $.
The IDBP method pairs $\tilde{\x}_k$ with $\tilde{\y}_k$, computed using (\ref{Eq_cost_func_our_y_sol}).
Note that (\ref{Eq_delta_cond2}) can be obtained simply by plugging $\tilde{\x}=\tilde{\x}_k$ and $\tilde{\y}=\tilde{\y}_k$ into (\ref{Eq_delta_cond}). Therefore, $\tilde{\x}_k$ and its associated $\tilde{\y}_k$ also violate the inequality in (\ref{Eq_delta_cond}). Finally, it is easy to see that $\tilde{\x}_k$ and $\tilde{\y}_k$ are feasible points of (\ref{Eq_cost_func_our}), since $\tilde{\x}_k$ is a feasible point of $s(\tilde{\x})$ and $\tilde{\y}_k$ satisfies $\H\tilde{\y}_k=\y$. Therefore, the condition in (\ref{Eq_delta_cond}) does not hold for all feasible $\tilde{\x}$ and $\tilde{\y}$, which means that $\delta=\tilde{\delta}$ violates it.
\end{proof}

Note that (\ref{Eq_delta_cond3}) can be easily evaluated for each iteration. Thus, violation of (\ref{Eq_delta_cond}) can be spotted (by violation of (\ref{Eq_delta_cond3})) and used for stopping the process, increasing $\delta$ and running the algorithm again. Of course, the opposite direction does not hold. Even when (\ref{Eq_delta_cond3}) is satisfied in all iterations, it does not guarantee satisfying (\ref{Eq_delta_cond}). However, the relaxed condition (\ref{Eq_delta_cond3}) provides an easy way to set $\delta$ with an approximation to the solution of (\ref{Eq_delta_cond}), which gives very good results in our experiments.

\subsection{IDBP for image inpainting}
\label{idbp_inpaint}

In the image inpainting problem, $\H$ is a selection of $m$ rows of $\I_n$. Therefore, $\H^\dagger = \H^T$, which is an $n \times m$ matrix that merely pads with $n-m$ zeros the vector on which it is applied. In this case, $\tilde{\y}_k$ is simply obtained by taking the observed pixels from $\y$ and the missing pixels from $\tilde{\x}_k$.
Moreover, setting $\delta$ according to Proposition \ref{proposition1} becomes ridiculously simple:
Since $\| \y-\H\tilde{\x}_k \|_2^2 = \| \H^\dagger ( \y - \H\tilde{\x}_k) \|_2^2$, it follows that $\delta=0$ satisfies (\ref{Eq_delta_cond3}) (with equality) for any $\tilde{\x}_k$. Obviously, if $\sigma_e=0$, a small positive $\delta$ is required in order to prevent the algorithm from getting stuck (because in this case $\sigma=\sigma_e+\delta=0$).

\subsection{IDBP for image deblurring}
\label{idbp_deblur}

In the image deblurring problem, for a circular shift-invariant blur operator whose kernel is $\h$, the computation of $\tilde{\y}_k$ can be efficiently implemented using Fast Fourier Transform (FFT). 
However, recall that in this case $\H$ is an ill-conditioned $n \times n$ matrix. 
Therefore, we replace $\H^\dagger$ with a regularized inversion of $\H$, using standard Tikhonov regularization, which is given in the Fourier domain by
\begin{align}
\label{Eq_Hinv_approx}
\tilde{\g} \triangleq \frac{ \mathcal{F}^* \{ \h \}  }{ | \mathcal{F} \{ \h \} |^2 + \epsilon \cdot \sigma_e^2 },
\end{align}
where $\mathcal{F} \{\cdot\}$ denotes the FFT operator, and $\epsilon$ is a parameter that controls the amount of regularization in the approximation of $\H^\dagger$.
Then, (\ref{Eq_cost_func_our_y_sol}) can be computed by
\begin{align}
\label{Eq_IDBP_fft}
\tilde{\y}_k = \mathcal{F}^{-1} \Big \{  \tilde{\g} \Big ( \mathcal{F}\{\y\} - \mathcal{F}\{\h\} \mathcal{F}\{\tilde{\x}_k\} \Big ) \Big \} + \tilde{\x}_k,
\end{align}
where $\mathcal{F}^{-1} \{\cdot\}$ denotes the inverse FFT operator.

Condition (\ref{Eq_delta_cond3}) can also be computed using FFT. Denoting the left-hand side (LHS) of (\ref{Eq_delta_cond3}) by $\eta_L$ and its right-hand side (RHS) by $\eta_R$, we have
\begin{align}
\label{Eq_LHS_RHS}
\eta_L &= \frac{1}{\sigma_e^2} \left \| \y- \mathcal{F}^{-1} \Big \{  \mathcal{F}\{\h\} \mathcal{F}\{\tilde{\x}_k\} \Big \} \right \|_2^2,  \nonumber \\
\eta_R &= \frac{1}{(\sigma_e+\delta)^2} \left \| \mathcal{F}^{-1} \Big \{  \tilde{\g} \Big ( \mathcal{F}\{\y\} - \mathcal{F}\{\h\} \mathcal{F}\{\tilde{\x}_k\} \Big ) \Big \} \right \|_2^2.
\end{align}

The deblurring version of IDBP includes two design parameters: $\delta$ and $\epsilon$, which also appear in the RHS of (\ref{Eq_delta_cond3}), i.e. in $\eta_R$ (note that $\tilde{\g}$ depends on $\epsilon$). Therefore, when applying IDBP with a given setting of $(\delta, \epsilon)$,  condition (\ref{Eq_delta_cond3}) can still be examined. Furthermore, note that in order to satisfy (\ref{Eq_delta_cond3}), its RHS can be decreased not only by increasing $\delta$, but also by increasing $\epsilon$ (which increases the denominator of $\tilde{\g}$).

We empirically observed that pairs of $(\delta, \epsilon)$ that give the best deblurring results indeed satisfy condition (\ref{Eq_delta_cond3}), while pairs of $(\delta, \epsilon)$ that lead to bad results often violate this condition. This behavior is demonstrated in Fig. \ref{fig:demonstation_psnr_cond} for {\em house} image in Scenario 1 (see Table \ref{table:blur_kernels} in Section \ref{exp_deblurring} for details about this scenario).
Fig. \ref{fig:demonstation_psnr} shows the PSNR of IDBP, with a plugged-in BM3D denoiser, as a function of the iteration number for several pairs of $(\delta, \epsilon)$. The LHS of (\ref{Eq_delta_cond3}) divided by its RHS (i.e. $\eta_L/\eta_R$) is presented in Fig.  \ref{fig:demonstation_cond} as a function of the iteration number. If this division is less than 1, even for a single iteration, it means that the original condition in (\ref{Eq_delta_cond}) is violated by the associated $(\delta, \epsilon)$. Recall that even when the division is higher than 1 for all iterations, it does not guarantee satisfying (\ref{Eq_delta_cond}). Therefore, a small margin should be kept. For example, the pair ($\delta$=5, $\epsilon$=7e-3), which reaches the highest PSNR in Fig. \ref{fig:demonstation_psnr}, has its smallest LHS/RHS ratio slightly below 3. 
When the margin further increases, graceful degradation in PSNR occurs, as observed for ($\delta$=7, $\epsilon$=7e-3) and ($\delta$=5, $\epsilon$=10e-3).

Equipped with the above observation, we suggest fixing $\delta$ (or $\epsilon$) and automatically tuning $\epsilon$ (or $\delta$) using condition (\ref{Eq_delta_cond3}) with some confidence margin. A scheme for IDBP with automatic tuning of  $\epsilon$ is presented in Algorithm \ref{Auto_IDBP_alg}. Starting with a small value of $\epsilon$, the ratio LHS/RHS of (\ref{Eq_delta_cond3}) is evaluated at the end of each IDBP iteration. If the ratio is smaller than a threshold $\tau$, then $\epsilon$ is slightly increased and IDBP is restarted. 
We do not check the ratio at the first iteration, as it strongly depends on the initial $\tilde{\y}_0$.
An alternative scheme that uses a fixed $\epsilon$ and gradually increases $\delta$ can be obtained in a similar way.
We noticed that the restarts in Algorithm \ref{Auto_IDBP_alg} happen in early iterations (e.g., restarts will occur at the second iteration for the bad initializations in Fig. \ref{fig:demonstation_cond}). Therefore, the proposed initialization scheme is not computationally demanding. 

The efficiency of the auto-tuned IDBP is demonstrated by improving the performance for the worst two initializations in Fig. \ref{fig:demonstation_psnr}, i.e. ($\delta$=2, $\epsilon$=7e-3) and ($\delta$=5, $\epsilon$=3e-3). For each of them, one parameter is kept as is and the second is auto-tuned using a threshold $\tau=3$. The results are shown in Figs. \ref{fig:demonstation_psnr_tuned} and \ref{fig:demonstation_cond_tuned}.

\begin{figure}[t]
  \centering
  \begin{subfigure}[b]{0.5\linewidth}
    \centering\includegraphics[width=120pt]{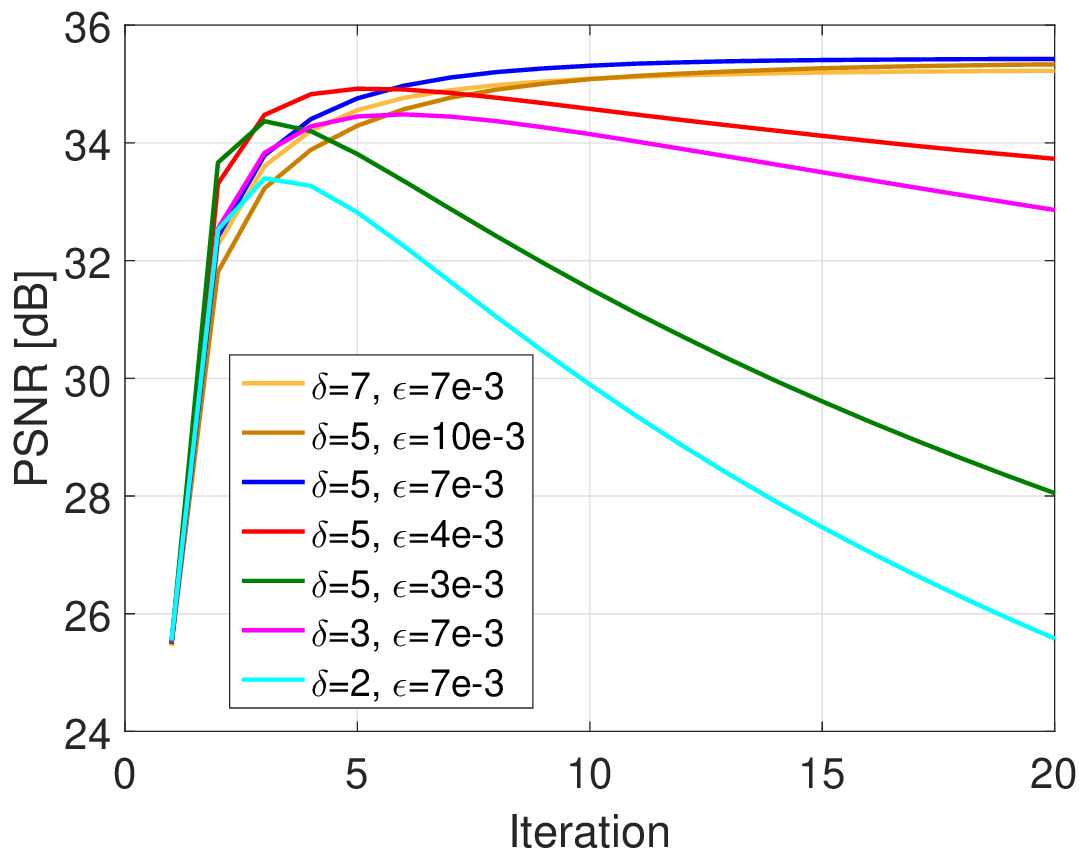}
    \caption{\label{fig:demonstation_psnr}}
  \end{subfigure}%
  \begin{subfigure}[b]{0.5\linewidth}
    \centering\includegraphics[width=120pt]{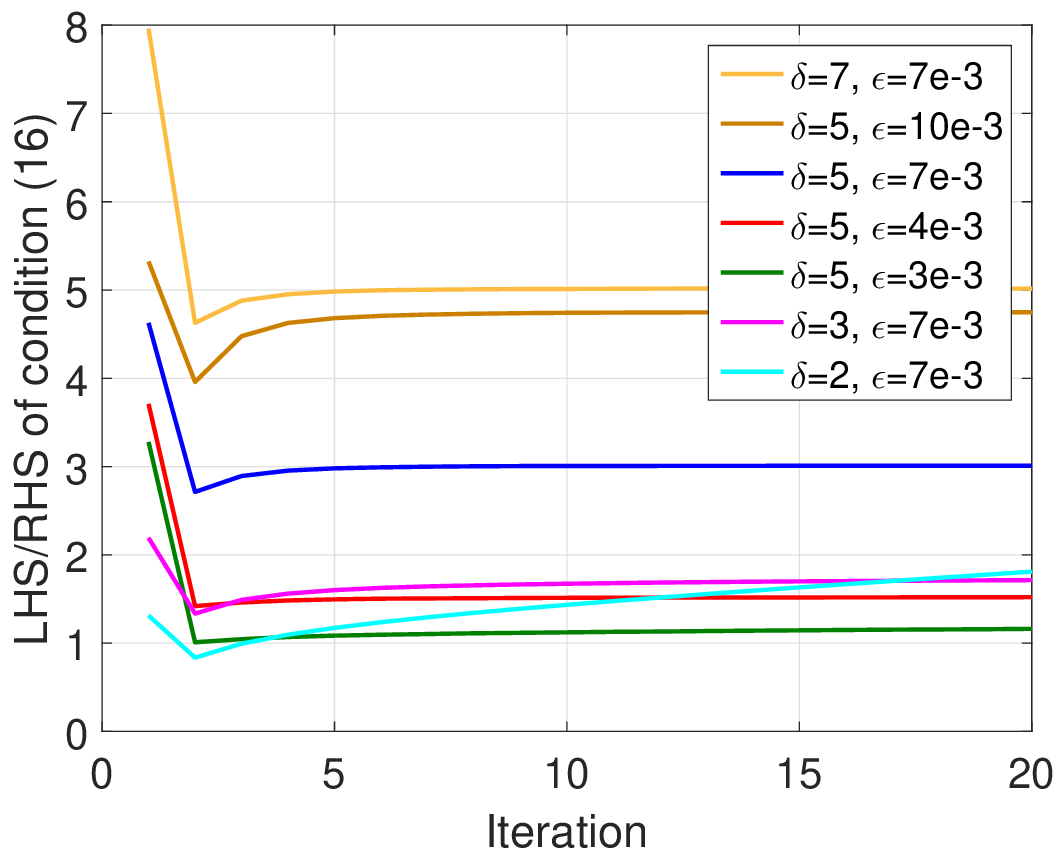}
    \caption{\label{fig:demonstation_cond}}
  \end{subfigure}
  \begin{subfigure}[b]{0.49\linewidth}
    \centering\includegraphics[width=120pt]{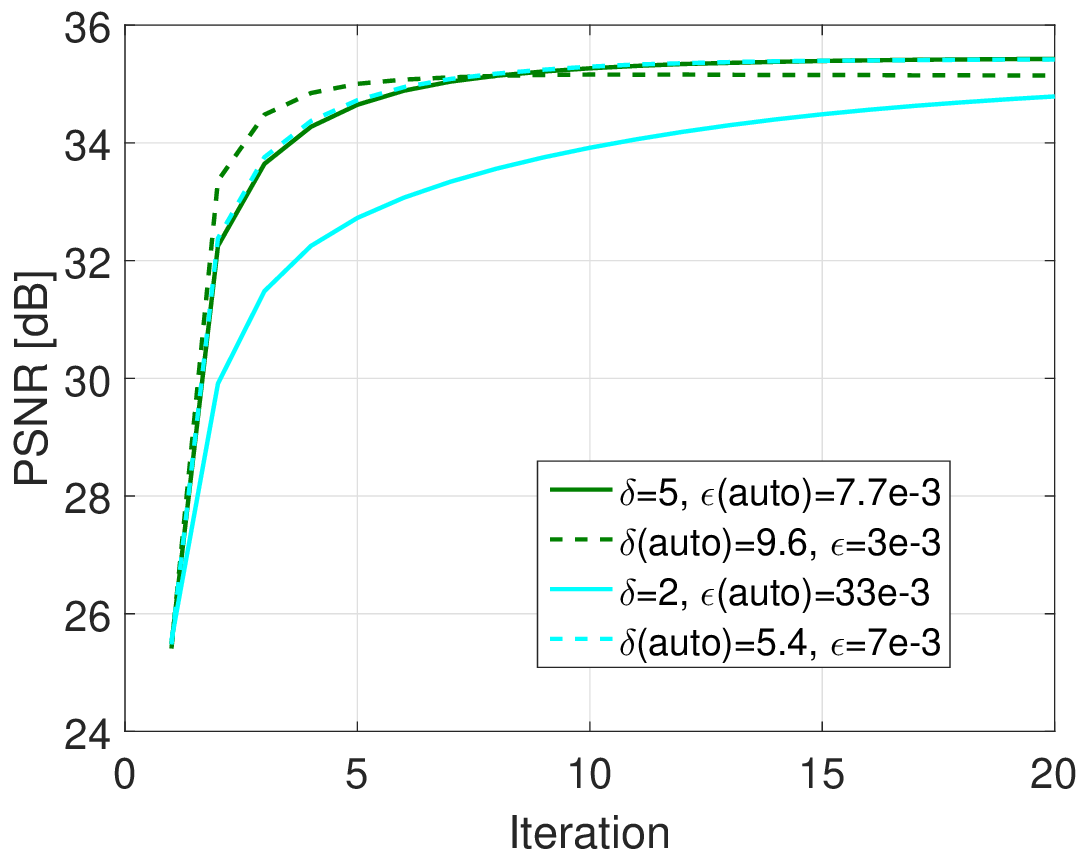}
    \caption{\label{fig:demonstation_psnr_tuned}}
  \end{subfigure}
  \begin{subfigure}[b]{0.49\linewidth}
    \centering\includegraphics[width=120pt]{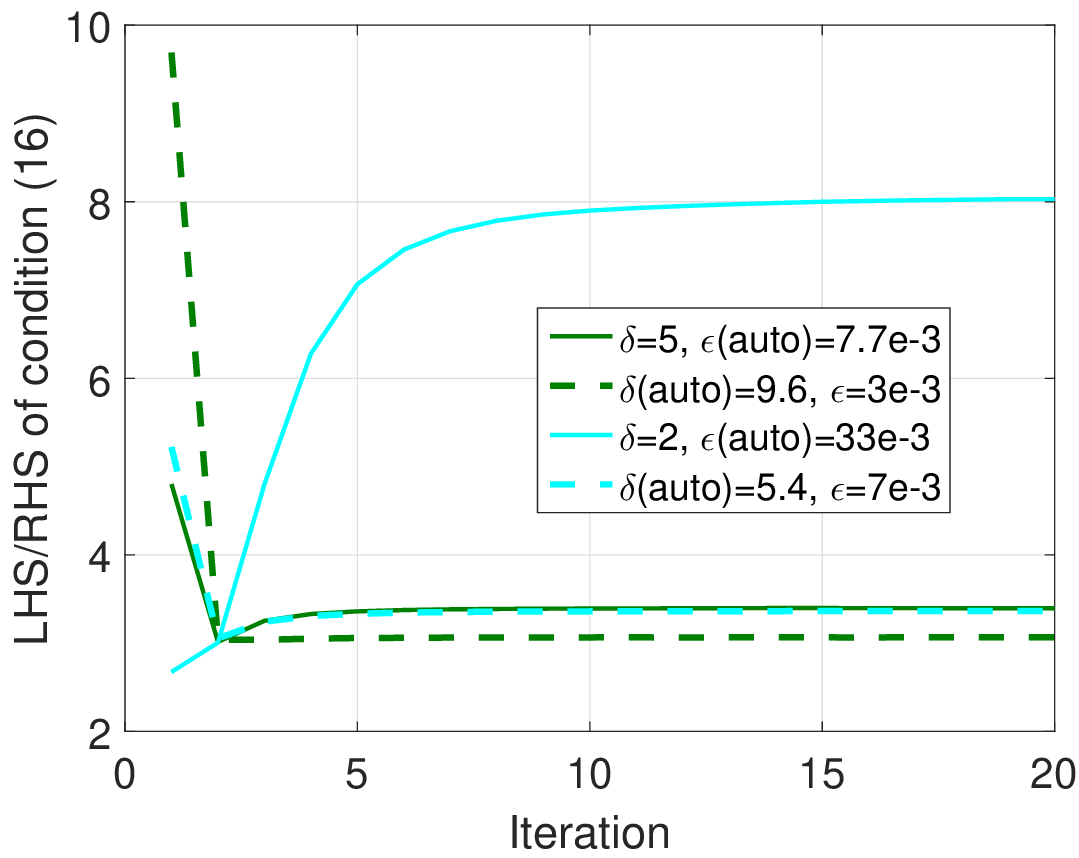}
    \caption{\label{fig:demonstation_cond_tuned}}
  \end{subfigure}
  \caption{(\subref{fig:demonstation_psnr}) IDBP deblurring results (PSNR vs. iteration number) for {\em house} in Scenario 1 for several pairs of $(\delta, \epsilon)$; (\subref{fig:demonstation_cond}) LHS of (\ref{Eq_delta_cond3}) divided by its RHS vs. iteration number. Note that if any iteration's value is less than 1, then the condition in (\ref{Eq_delta_cond}) is violated. Since the opposite direction does not hold, it is preferable to keep a margin above 1; (\subref{fig:demonstation_psnr_tuned}) the results of the auto-tuned IDBP initialized with the values of $(\delta, \epsilon)$ that give the two worst results in (\subref{fig:demonstation_psnr}); (\subref{fig:demonstation_cond_tuned}) LHS of (\ref{Eq_delta_cond3}) divided by its RHS after auto-tuning.}
\label{fig:demonstation_psnr_cond}
\end{figure}

\begin{algorithm}
\caption{Auto-tuned IDBP for deblurring}
\vspace{2mm}
\kwInput{$\h, \y, \sigma_e$, denoising operator $\mathcal{D}(\cdot;\sigma)$, stopping criterion. $\y = \x*\h+\e$, such that $\e \sim \mathcal{N}(\0,\sigma_e^2\I_n)$ and $\x$ is an unknown signal whose prior model is specified by $\mathcal{D}(\cdot;\sigma)$.}
\kwOutput{$\hat{\x}$ an estimate for $\x$.}
\kwParameters{$\tilde{\y}_0=$ some initialization, $k=0$, $\delta=$ moderate fixed value, $\epsilon=$ small initial value, $\Delta\epsilon=$ small increment, $\tau=$ confidence margin greater than 1.}
\kwDefaults{$\tilde{\y}_0=\y$, $\delta=5$, $\epsilon=$5e-4, $\Delta\epsilon$=1e-4, $\tau=3$.}
\While{stopping criterion not met}{
    $k = k+1$\;
    $\tilde{\x}_k=\mathcal{D}(\tilde{\y}_{k-1};\sigma_e+\delta)$\;
    Compute $\tilde{\y}_k$ using (\ref{Eq_IDBP_fft}) (note that $\tilde{\g}$ depends on $\epsilon$)\;
    Compute $\eta_L$ and $\eta_R$ using (\ref{Eq_LHS_RHS}) (i.e. LHS and RHS of (\ref{Eq_delta_cond3})\;
     \uIf{$k>1$ and $\eta_L/\eta_R<\tau$}
        {
          $\epsilon = \epsilon + \Delta\epsilon$\;
          Restart process: $k = 0$\;
    }
}
$\hat{\x} = \tilde{\x}_k$\;
\label{Auto_IDBP_alg}
\end{algorithm}

\section{Mathematical Analysis of the Algorithm}
\label{sec_analysis}

\subsection{Analysis of the sequence $\{\tilde{\y}_k\}$}
\label{sec_sequence_yk}

The IDBP algorithm creates the sequence $\{\tilde{\y}_k\}$ that can be interpreted as a sequence of updated measurements. It is desired that $\tilde{\y}_k$ is improved with each iteration, i.e. that $\tilde{\x}_{k+1}$, obtained from $\tilde{\y}_k$, estimates $\x$ better than $\tilde{\x}_k$, which is obtained from $\tilde{\y}_{k-1}$.

Assuming that the result of the denoiser, denoted by $\overline{\x}$, is perfect, i.e. $\overline{\x}=\x$, we get from (\ref{Eq_cost_func_our_y_sol})
\begin{align}
\label{Eq_y_overline}
\overline{\y} &= \H^\dagger\y + (\I_n - \H^\dagger\H)\overline{\x}  \nonumber \\
&= \H^\dagger (\H\x+\e) + (\I_n - \H^\dagger\H)\x  \nonumber \\
&= \x + \H^\dagger \e.
\end{align}
The last equality describes a model that has only noise (possibly colored), and is much easier to deal with than the original model (\ref{Eq_general_model}). Therefore, $\overline{\y}$ can be considered as the optimal improved measurements that our algorithm can achieve. As we wish to make no specific assumptions on the denoising scheme $\tilde{\x}_k=\mathcal{D}(\tilde{\y}_{k-1};\sigma)$, improvement of $\{\tilde{\y}_k\}$ will be measured by the Euclidean distance to $\overline{\y}$.

Denote by $\P_H \triangleq \H^\dagger\H$ the orthogonal projection onto the row space of $\H$, and its orthogonal complement by $\Q_H \triangleq \I_n-\H^\dagger\H$. 
The updated measurements $\tilde{\y}_k$ are always consistent with $\y$ on $\P_H\x$, and do not depend on $\P_H\tilde{\x}_k$, as can be seen from
\begin{align}
\label{Eq_y_tilde_b}
\tilde{\y}_k &= \H^\dagger (\H\x+\e) + \Q_H \tilde{\x}_k  \nonumber \\
&= \P_H\x + \H^\dagger \e + \Q_H \tilde{\x}_k.
\end{align}
Thus, the following theorem ensures that iteration $k$ improves the results, provided that $\tilde{\x}_k$ is closer to $\x$ than $\tilde{\y}_{k-1}$ on the null space of $\H$, i.e.,
\begin{align}
\label{Eq_denoiser_assump_thm}
\| \Q_H(\tilde{\x}_k - \x) \|_2 < \| \Q_H(\tilde{\y}_{k-1} - \x) \|_2.
\end{align}

\begin{theorem}
\label{theorem1}
Assuming that (\ref{Eq_denoiser_assump_thm}) holds at the $k$th iteration of IDBP, then we have
\begin{align}
\label{Eq_theorem1}
\| \tilde{\y}_k - \overline{\y} \|_2 < \| \tilde{\y}_{k-1} - \overline{\y} \|_2.
\end{align}
\end{theorem}

\begin{proof}
Note that
\begin{align}
\label{Eq_proof1}
\Q_H \tilde{\y}_{k-1} = \Q_H ( \H^\dagger \y + \Q_H \tilde{\x}_{k -1} ) = \Q_H \tilde{\x}_{k -1}.
\end{align}
We obtain (\ref{Eq_theorem1}) by
\begin{align}
\label{Eq_proof2}
\| \tilde{\y}_k - \overline{\y} \|_2  & = \| (\P_H\x + \H^\dagger \e + \Q_H \tilde{\x}_k) - (\x + \H^\dagger \e) \|_2 \nonumber \\
& = \| \Q_H (\tilde{\x}_k - \x) \|_2 \nonumber \\
& < \| \Q_H (\tilde{\x}_{k-1} - \x) \|_2 \nonumber \\
& = \| (\P_H\x + \H^\dagger \e + \Q_H \tilde{\x}_{k-1}) - (\x + \H^\dagger \e) \|_2 \nonumber \\
& = \| \tilde{\y}_{k-1} - \overline{\y} \|_2,
\end{align}
where the inequality follows from (\ref{Eq_denoiser_assump_thm}) and (\ref{Eq_proof1}).
\end{proof}

A denoiser that makes use of a good prior (and suitable $\sigma$) is expected to satisfy (\ref{Eq_denoiser_assump_thm}), at least in early iterations. For example, in the inpainting problem $\Q_H$ is associated with the missing pixels, and in the deblurring problem $\Q_H$ is associated with the data that suffer the greatest loss by the blur kernel. Therefore, in both cases $\Q_H\tilde{\x}_{k}$ is expected to be closer to $\Q_H\x$ than $\Q_H\tilde{\y}_{k-1}$.
Note that if (\ref{Eq_denoiser_assump_thm}) holds for all iterations, then Theorem \ref{theorem1} ensures monotonic improvement and convergence of $\{\tilde{\y}_k\}$, and thus, a fixed point convergence of IDBP.
However, note that it does not guarantee that $\overline{\y}$ is the limit of the sequence $\{\tilde{\y}_k\}$.

\subsection{Recovery guarantees}
\label{sec_guarantees}

Similar to P\&P, in order to prove more than a fixed point convergence of IDBP, strict assumptions on the denoising scheme are required. For global convergence of P\&P, it is enough to assume that the denoiser is non-expansive and has a symmetric gradient \cite{sreehari2016plug}, which allows using the proximal mapping theorem of Moreau \cite{moreau1965proximite}. However, non-expansiveness property of a denoiser is very demanding, as it requires that for a given noise level $\sigma$ we have
\begin{align}
\label{Eq_non_expansive}
\| \mathcal{D}(\z_1;\sigma) - \mathcal{D}(\z_2;\sigma) \|_2 \leq K_{\sigma} \| \z_1 - \z_2 \|_2,
\end{align}
for any $\z_1$ and $\z_2$ in $\Bbb R^n$, with $K_{\sigma} \leq 1$.

In this work we take a different route that exploits the structure of the IDBP algorithm, where the denoiser's output is always projected onto the null space of $\H$. Instead of assuming (\ref{Eq_denoiser_assump_thm}), we use the following assumptions:

\begin{condition}
\label{cond1}
The denoiser is bounded, in the sense of
\begin{align}
\label{Eq_bounded}
\| \mathcal{D}(\z;\sigma) - \z \|_2 \leq \sigma B,
\end{align}
for any $\z \in \Bbb R^n$, where $B$ is a universal constant independent of $\sigma$.
\end{condition}

\begin{condition}
\label{cond3}
For a given noise level $\sigma>0$, the projection of the denoiser onto the null space of $\H$ is a contraction, i.e., it satisfies
\begin{align}
\label{Eq_contraction}
\| \Q_H \mathcal{D}(\z_1;\sigma) - \Q_H \mathcal{D}(\z_2;\sigma) \|_2 \leq K_{\sigma} \| \z_1 - \z_2 \|_2,
\end{align}
for any $\z_1 \neq \z_2$ in $\Bbb R^n$, where $K_{\sigma} < 1$, and $\Q_H \triangleq \I_n-\H^\dagger\H$.
\end{condition}

Condition \ref{cond1} implies that $\mathcal{D}(\z;0)=\z$, as can be expected from a denoiser. Thus, it prevents considering a trivial mapping, e.g. $\mathcal{D}(\z;\sigma)=\0$ for all $\z$, which trivially satisfies Condition \ref{cond3}. Regarding the second condition, even though it describes a contraction, it considers the operator $\Q_H \mathcal{D}(\cdot;\sigma)$. Therefore, for some cases of $\H$, it might be weaker than non-expansiveness of $\mathcal{D}(\cdot;\sigma)$. Our main recovery guarantee is given in the following theorem.

\begin{theorem}
\label{theorem2}
Let $\y=\H\x+\e$, apply IDBP with some $\sigma>0$ for the denoising operation, and assume that Condition \ref{cond1} holds. Assume also that Condition \ref{cond3} holds for this choice of $\sigma$. Then, with the notation of IDBP we have
\begin{align}
\label{Eq_theorem2}
\| \tilde{\x}_{k+1} - \x \|_2 \,\, \leq \,\, K_{\sigma}^k \| \tilde{\y}_0 - \overline{\y} \|_2  + \frac{1}{1-K_{\sigma}} \| \H^\dagger \e \|_2 + C_\sigma,
\end{align}
where $\overline{\y} = \x + \H^\dagger \e$ and $C_\sigma \triangleq (\frac{1}{1-K_\sigma}+5)\sigma B$.
\end{theorem}

The proof of Theorem \ref{theorem2} appears in Appendix \ref{app:thm}.

Theorem \ref{theorem2} provides an upper bound on the error of IDBP w.r.t. the true signal $\x$. 
Despite the fact that Condition \ref{cond3} may not be verified for the widely-used denoisers (similar to the non-expansiveness condition required for convergence of P\&P), the bound is useful because it demonstrates the effect of different parameters on the convergence rate and accuracy. 
The implications of the bound are based on the observation that Condition \ref{cond3} implies inverse proportion between $\sigma$ and $K_\sigma$. To see this, note that the smaller $\sigma$ is, the smaller is the effect of the denoiser on its input. Therefore, since (\ref{Eq_contraction}) needs to be satisfied for any two signals $\z_1$ and $\z_2$, a larger $K_\sigma$ is required. 

Equipped with this observation, from the first term in the bound it can be seen that applying IDBP with a relatively large $\sigma$ is expected to accelerate its convergence, since $K_\sigma$ is smaller in this case. Decreasing $K_\sigma$ also reduces the second term. However, this term may be an artifact of our proof. Interestingly, the third term $C_\sigma$ suggests using IDBP with the smallest possible $\sigma$, for which the increasing effect on $K_\sigma$ is small and can be compensated by using more iterations. 
Therefore, in order to obtain a more accurate result, and assuming no restrictions on the number of iterations, smaller $\sigma$ that reduces $C_\sigma$ is beneficial.
The last observation agrees with our suggestion to choose $\delta$ according to (\ref{Eq_delta_cond}), where $\sigma=\sigma_e+\delta$ is minimized (under a constraint that aims to prevent losing solutions when (\ref{Eq_cost_func_our}) is being solved instead of (\ref{Eq_cost_func3})).

To the best of our knowledge, there is no equivalent result like Theorem \ref{theorem2} for P\&P, as its existing convergence guarantees refer to approaching a minimizer of the original cost function (\ref{Eq_cost_func1}), which is not necessarily identical to $\x$. Therefore, even though we propose an alternative method to minimize (\ref{Eq_cost_func1}), we choose to consider IDBP error w.r.t. the true $\x$. Note though that the proof technique we show here can be also used to bound the Euclidean distance between the IDBP estimation and a (pre-computed) solution of (\ref{Eq_cost_func1}) with only minor technical changes.

\section{Experiments}
\label{sec_Experiments}

We examine the performance of IDBP for two test scenarios: the inpainting and the deblurring problems.\footnote{Matlab code available at https://github.com/tomtirer/IDBP.} We compare the IDBP performance to P\&P and to the recent state-of-the-art method IRCNN \cite{zhang2017learning}. The IRCNN is based on a similar idea as P\&P. It trains a set of 25 denoising neural networks (DNNs), each for a different noise level, and plugs them into a quadratic penalty method scheme \cite{nocedal2006sequential} to minimize (\ref{Eq_pnp_cost_func2}). 
Such a scheme requires increasing the penalty parameter between iterations, which is translated to using about two dozen DNNs for each inverse problem. 

For the comparison of IDBP and P\&P we use BM3D \cite{dabov2007image} as the denoising algorithm, and denote the resulting methods by IDBP-BM3D and P\&P-BM3D, respectively.  For a fair comparison between IDBP and IRCNN, we plug the trained DNNs of the latter into our IDBP scheme, and denote the resulting method by IDBP-CNN. We emphasize that {\em IDBP-CNN requires a single DNN for each inverse problem}, as we do not modify $\delta$ between iterations.

In addition, for each one of the two problems: inpainting and deblurring, we compare the performance of the above methods to another algorithm that has been specially tailored for that problem \cite{danielyan2012bm3d}, \cite{ram2013image}.
We use the following eight test images in all experiments: {\em cameraman}, {\em house}, {\em peppers}, {\em Lena}, {\em Barbara}, {\em boat}, {\em hill} and {\em couple}. We also report average results on BSD68 dataset, which includes 68 grayscale images of size 481$\times$321 pixels. This dataset is used in many works, e.g. \cite{zhang2017learning, roth2005fields}.

\subsection{Image inpainting}

In the image inpainting problem, $\H$ is a selection of $m$ rows of $\I_n$ and $\H^\dagger = \H^T$, which simplifies both P\&P and IDBP. In P\&P, the first step can be solved for each pixel individually. In IDBP, $\tilde{\y}_k$ is obtained merely by taking the observed pixels from $\y$ and the missing pixels from $\tilde{\x}_k$. 
Note that the computational cost of each iteration of P\&P-BM3D and IDBP-BM3D is of the same scale, dominated by the complexity of the BM3D denoiser. 
Therefore, the overall complexity of P\&P-BM3D and IDBP-BM3D can be compared by the number of iterations each technique is using.
Similarly, the computational cost of each iteration of IRCNN and IDBP-CNN is of the same scale, and their overall complexity is determined by the number of iterations that they use.
For all methods we use the result of a simple median scheme as their initialization (e.g. for $\check{\v}_0$ in P\&P and for $\tilde{\y}_0$ in IDBP). 

The first experiment demonstrates the performance of IDBP, P\&P, IRCNN and inpainting based on Image Processing using Patch Ordering (IPPO) approach \cite{ram2013image}, for the noiseless case ($\sigma_e = 0$) with 80\% missing pixels, selected at random. For {\em peppers}, the first and last rows and columns, which contain defective intensity values are ignored, as they damage the quality assessment obtained by PSNR.
For IPPO and IRCNN we use the code supplied by the authors\footnote{Downloaded from \url{http://www.cs.technion.ac.il/~elad/Various/IPPOBox.zip}, and \url{https://github.com/cszn/IRCNN}.}, where the same scenario is examined.
The parameters of P\&P-BM3D are optimized for best reconstruction quality. We use $\beta=1$, $\lambda=10/255$ and 150 iterations, and also set the noise standard deviation to 0.001, i.e. nonzero, in order to compute $\check{\x}_k$.

Considering IDBP, in Section \ref{idbp_inpaint}, it is suggested that $\delta=0$. However, since in this case $\sigma_e+\delta=0$, a small positive $\delta$, e.g. $\delta=1$, is required. Indeed, this setting gives good performance, but also requires many more iterations than the other methods. Therefore, we use an alternative approach. We use a larger value for $\delta$ but take the last $\tilde{\y}_k$ as the final estimate, which is equivalent to performing the last denoising with the recommended $\delta=0$. We set $\delta=5$ for IDBP-BM3D, which allows us to use only 150 iterations (same as P\&P-BM3D), and $\delta=10$ for IDBP-CNN, which requires only 30 iterations (same as IRCNN).
Fig. \ref{fig:IDBP_inpainting_delta1_5} shows the results of IDBP-BM3D for the {\em house} image. It approves that the alternative implementation performs well and requires significantly less iterations (note that the x-axis has a logarithmic scale). Therefore, for the comparison of the different inpainting methods in this experiment (where $\sigma_e=0$), we use the alternative implementation of IDBP. 
The empirical behavior observed here, agrees with the theoretical observation at the end of Section \ref{sec_guarantees}: larger $\sigma=\sigma_e+\delta$ requires less iterations (due to smaller $K_\sigma$) but results in higher error. Note also that it is possible to decrease $\delta$ as the iterations increase. However, in this work {\em we aim at demonstrating the performance of the IDBP scheme with minimal parameter tuning as possible}.

The results (PSNR and SSIM \cite{wang2004image}) of the algorithms are given in Table \ref{table:noiselss_inpainting}. In these experiments, the BM3D prior outperforms the learned one. IDBP-BM3D is usually better than IPPO, but slightly inferior to P\&P-BM3D. This is the cost of accelerating IDBP by setting $\delta$ to a value which is significantly larger than zero. However, this observation also hints that IDBP may shine for noisy measurements, where $\delta=0$ can be used without increasing the number of iterations. We also remark that IPPO gives the best results for {\em Barbara} because in this image P\&P-BM3D and IDBP-BM3D require more than the fixed 150 iterations.

\begin{figure}
    \centering
    \includegraphics[width=0.35\textwidth]{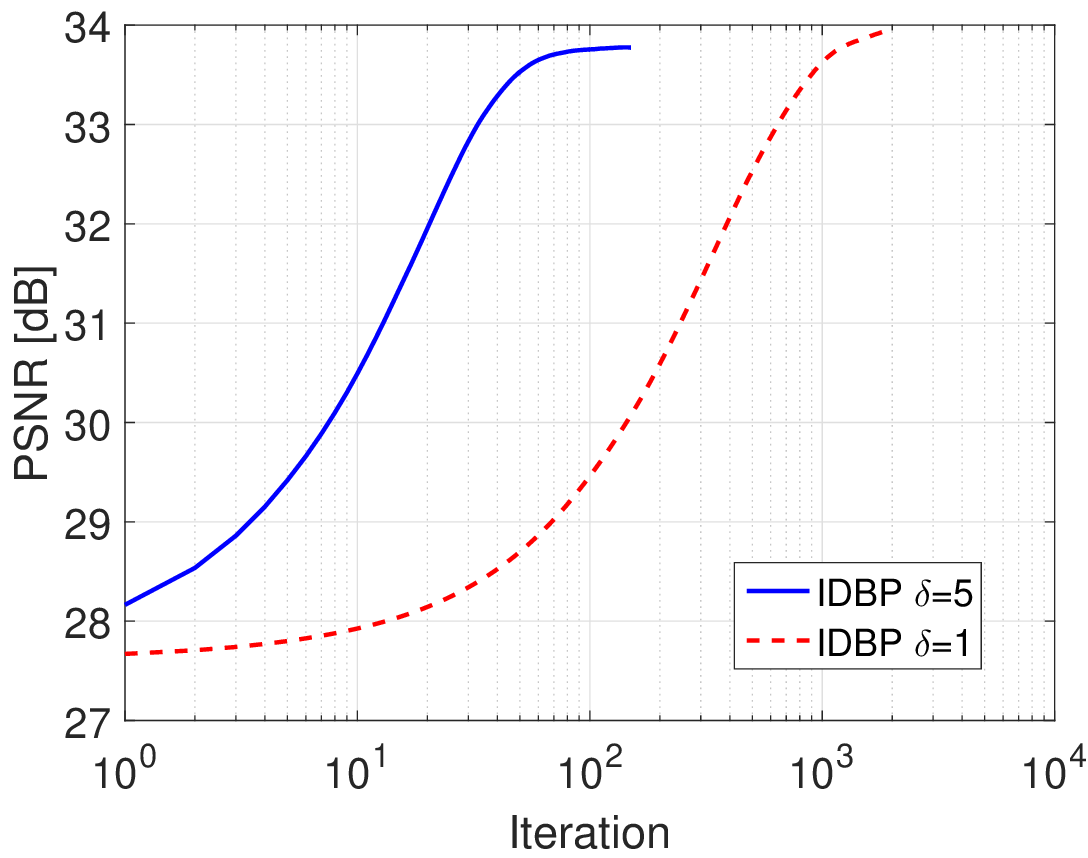}
    \caption{IDBP-BM3D recovery (PSNR vs. iteration) of {\em house} test image with 80\% missing pixels and no noise.} \label{fig:IDBP_inpainting_delta1_5}
\end{figure}

\begin{table*}
\renewcommand{\arraystretch}{1.3}
\caption{Inpainting results (PSNR in dB / SSIM) for 80\% missing pixels and $\sigma_e=0$.} \label{table:noiselss_inpainting}
\centering
    \begin{tabular}{ | l | l | l | l | l | l | l | l | l |}
    \hline
            & {\em cameraman} & {\em house} & {\em peppers} & {\em Lena} & {\em Barbara} & {\em boat} & {\em hill} & {\em couple} \\ \hline
    IPPO & 24.78 / 0.832 & 32.64 / 0.909  & {28.22} / 0.882 & 31.84 / 0.895 & {\bfseries 29.89} / {\bfseries 0.906} & 28.17 / 0.822  & 29.47 / 0.815 & 28.22 / 0.842 \\ \hline
    P\&P-BM3D & 24.83 / {\bfseries 0.845} & {\bfseries 34.72} / {\bfseries 0.920}  & {\bfseries 28.77} / {\bfseries 0.895} & {\bfseries 32.41} / {\bfseries 0.903} & 25.68 / 0.862 & {\bfseries 28.83} / {\bfseries 0.844}  & {\bfseries 29.95} / {\bfseries 0.831} & {\bfseries 29.01} / {\bfseries 0.865} \\ \hline
    IRCNN & {\bfseries 25.27} / 0.838	&32.21 / 0.888	&28.26 / 0.882	&31.56 / 0.889	&27.34 / 0.858	&27.88 / 0.809	&29.24 / 0.804	&28.23 / 0.834  \\ \hline
    IDBP-BM3D & {24.86} / 0.840 & 33.78 / 0.893  & 28.58 / 0.885  & 32.13 / 0.893 & 25.55 / 0.841 & 28.51 / 0.824  & 29.74 / 0.810 & 28.80 / 0.846 \\ \hline
    IDBP-CNN & 24.24 / 0.826	&32.14 / 0.881	&27.80 / 0.866	&31.22 / 0.880	&24.29 / 0.796	&27.72 / 0.803	&29.01 / 0.790	&27.98 / 0.818
 \\
    \hline
    \end{tabular}
\end{table*}

The second experiment demonstrates the performance of IDBP, P\&P and IRCNN with 80\% missing pixels, as before, but this time $\sigma_e=10$. Noisy inpainting has not been implemented yet by IPPO \cite{ram2013image}. 
The parameters of P\&P-BM3D that give us the best results are $\beta=0.8$, $\lambda=5/255$ and 150 iterations. Using the same parameter values as before deteriorates the performance significantly. For IRCNN we just update the new $\sigma_e$ in the code (our further tuning efforts have not been successful). Contrary to P\&P, in this experiment tuning the parameters of both IDBP versions can be avoided. We follow Section \ref{idbp_inpaint} and set $\delta=0$. Moreover, IDBP-BM3D now requires only 75 iterations, half the number of P\&P-BM3D. For IDBP-CNN we still use only 30 iterations (same as IRCNN).
The results are given in Table \ref{table:noisy_inpainting1}.
P\&P-BM3D is slightly inferior to IDBP-BM3D, despite having twice the number of iterations and a burdensome parameter tuning.  IDBP-CNN and IRCNN, which have similar computational cost, are highly competitive. However, note that IRCNN uses two dozen different DNNs for each inverse problem, while IDBP requires only a single DNN, as we do not modify $\delta$ between iterations.
The results for {\em house} are also presented in Fig. \ref{inpainting_example}. In this case, the BM3D-based methods obtain better visual results than their DNN-based alternatives. Furthermore, P\&P-BM3D reconstruction has slightly more artifacts than IDBP-BM3D (e.g. ringing artifacts near the right window), and IDBP-CNN result is smoother than IRCNN, yet recovers finer details (e.g. the black pipe on the roof).

\begin{figure}
\captionsetup[subfigure]{labelformat=empty}
  \centering
  \begin{subfigure}[b]{0.5\linewidth}
    \centering\includegraphics[width=120pt]{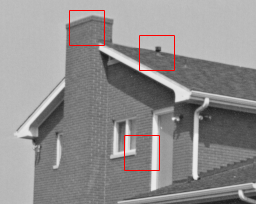}\\
\vspace{1mm}
    \centering\includegraphics[width=37.5pt]{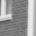}
    \centering\includegraphics[width=37.5pt]{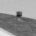}
    \centering\includegraphics[width=37.5pt]{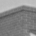}
    \caption{Original image}
\vspace{1mm}
  \end{subfigure}%
  \begin{subfigure}[b]{0.5\linewidth}
    \centering\includegraphics[width=120pt]{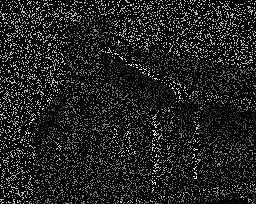}\\
\vspace{1mm}
    \centering\includegraphics[width=37.5pt]{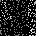}
    \centering\includegraphics[width=37.5pt]{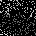}
    \centering\includegraphics[width=37.5pt]{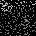}
    \caption{Subsampled and noisy image}
\vspace{1mm}
  \end{subfigure}
  \begin{subfigure}[b]{0.5\linewidth}
    \centering\includegraphics[width=120pt]{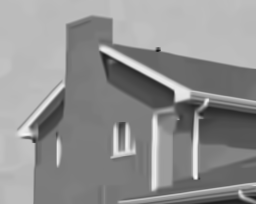}\\
\vspace{1mm}
    \centering\includegraphics[width=37.5pt]{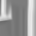}
    \centering\includegraphics[width=37.5pt]{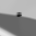}
    \centering\includegraphics[width=37.5pt]{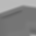}
    \caption{P\&P-BM3D (31.53 dB)}
\vspace{1mm}
  \end{subfigure}%
  \begin{subfigure}[b]{0.5\linewidth}
    \centering\includegraphics[width=120pt]{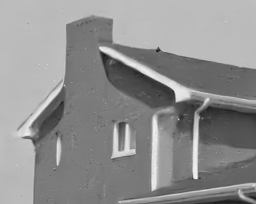}\\
\vspace{1mm}
    \centering\includegraphics[width=37.5pt]{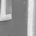}
    \centering\includegraphics[width=37.5pt]{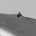}
    \centering\includegraphics[width=37.5pt]{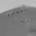}
    \caption{IRCNN (30.61 dB)}
\vspace{1mm}
  \end{subfigure}
  \begin{subfigure}[b]{0.5\linewidth}
    \centering\includegraphics[width=120pt]{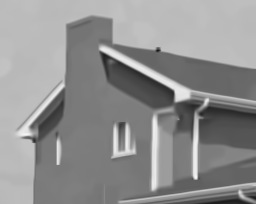}\\
\vspace{1mm}
    \centering\includegraphics[width=37.5pt]{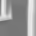}
    \centering\includegraphics[width=37.5pt]{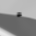}
    \centering\includegraphics[width=37.5pt]{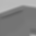}
    \caption{IDBP-BM3D (31.62 dB)}
  \end{subfigure}%
  \begin{subfigure}[b]{0.5\linewidth}
    \centering\includegraphics[width=120pt]{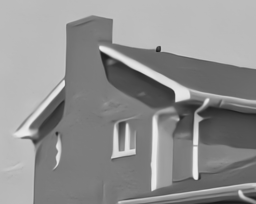}\\
\vspace{1mm}
    \centering\includegraphics[width=37.5pt]{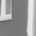}
    \centering\includegraphics[width=37.5pt]{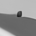}
    \centering\includegraphics[width=37.5pt]{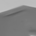}
    \caption{IDBP-CNN (31.16 dB)}
  \end{subfigure}
  \caption{Recovery of {\em house} image with 80\% missing pixels and $\sigma_e=10$. From left to right and from top to bottom: original image, subsampled and noisy image, reconstruction of P\&P-BM3D, reconstruction of IRCNN, reconstruction of the proposed IDBP-BM3D, and reconstruction of the proposed IDBP-CNN.}
\label{inpainting_example}
\end{figure}

\begin{table*}
\renewcommand{\arraystretch}{1.3}
\caption{Inpainting results (PSNR in dB / SSIM) for 80\% missing pixels and $\sigma_e=10$.} \label{table:noisy_inpainting1}
\centering
    \begin{tabular}{ | l | l | l | l | l | l | l | l | l |}
    \hline
            & {\em cameraman} & {\em house} & {\em peppers} & {\em Lena}& {\em Barbara} & {\em boat} & {\em hill} & {\em couple}  \\ \hline
    P\&P-BM3D & 24.55 / 0.785 & 31.53 / 0.848  & {27.21} / 0.827 & 30.10 / 0.833 & 24.45 / 0.735 & 27.01 / 0.731  & 27.94 / 0.706 & {\bfseries 27.23} / 0.761  \\ \hline
    IRCNN & {\bfseries 24.75} / {\bfseries 0.794}	& 30.61 / 0.845	&{\bfseries 27.25} / 0.827 	&29.94 / 0.833	&{\bfseries 25.94} / {\bfseries 0.781}	&26.86 / 0.741	&27.90 / 0.726	&26.98 / 0.757 \\ \hline
    IDBP-BM3D & {24.68} / 0.786 & {\bfseries 31.62} / 0.850  & 27.24 / 0.829  & {30.14} / 0.835 & {25.03} / 0.755 & {\bfseries 27.02} / 0.731  & {\bfseries 28.00} / 0.708  & 27.22 / 0.759 \\ \hline
    IDBP-CNN & 23.94 / 0.791	&31.16 / {\bfseries 0.851}	& {\bfseries 27.25} / {\bfseries 0.841} 	& {\bfseries 30.17} / {\bfseries 0.849}	&23.62 / 0.753	&26.95 / {\bfseries 0.756}	&27.93 / {\bfseries 0.734}	&27.04 / {\bfseries 0.773}
 \\
    \hline
    \end{tabular}
\end{table*}

We repeat the last experiment with slightly increased noise level of $\sigma_e=12$, but still use the same parameter tuning for all methods (e.g. P\&P-BM3D uses $\beta=0.8$ and $\lambda=5/255$, which are optimized for $\sigma_e=10$). This situation is often encountered in practice, when calibrating a system for all possible scenarios is impossible. The results are given in Table \ref{table:noisy_inpainting3}. In this case, IDBP-CNN is usually better than IRCNN, and IDBP-BM3D clearly outperforms P\&P-BM3D. This experiment shows another advantage of our inpainting algorithm over P\&P, as it is less sensitive to parameter tuning.
The results for {\em peppers} are presented in Fig. \ref{inpainting_example2}. This time the DNN-based methods exhibit better results. Moreover, the IDBP scheme leads to improved reconstructions for the two types of denoisers.

Next, we demonstrate the application of the methods for removal of superimposed text. Fig. \ref{inpainting_example3} displays the results for {\em Lena} with $\sigma_e=10$. IDBP-CNN outperforms the other methods in this case (e.g. it recovers the left eye better). Note that in this example the ratio of missing pixels is significantly lower than in the previous experiments. Therefore, when we repeat it without noise, all methods give very good results without noticeable visual differences. In the noiseless case, PSNR values (in dB) of 37.90, 37.12, 37.86, 37.37 are obtained for P\&P-BM3D, IRCNN, IDBP-BM3D, and IDBP-CNN, respectively.

Finally, we examine the performance of the methods on BSD68 dataset. We repeat the first three experiments with the same algorithms and settings as before. The only difference is that we reduce the ratio of missing pixels to 50\% (for 80\% missing pixels all methods perform poorly, e.g. the average SSIM is lower than 0.8 even in the noiseless case). Also, the performance of IPPO is not reported since its code does not support the dimension of BSD68 images.
The results are given in Table \ref{table:bsd68_inpainting}. As before, the IDBP-based algorithms demonstrate competitive performance in the noiseless case and improved results for the noisy settings, while maintaining implementation advantages over other methods (IDBP-BM3D requires less parameter tuning and iterations than P\&P-BM3D, and IDBP-CNN requires only a single trained DNN per scenario, compared to the 25 DNNs required by IRCNN).

\begin{table*}
\renewcommand{\arraystretch}{1.3}
\caption{Inpainting results (PSNR in dB / SSIM) for 80\% missing pixels and $\sigma_e=12$, with the same parameters of Table \ref{table:noisy_inpainting1} (tuned for $\sigma_e=10$).} \label{table:noisy_inpainting3}
\centering
    \begin{tabular}{ | l | l | l | l | l | l | l | l | l |}
    \hline
            & {\em cameraman} & {\em house} & {\em peppers} & {\em Lena} & {\em Barbara} & {\em boat} & {\em hill} & {\em couple} \\ \hline
    P\&P-BM3D & 24.43 / 0.774 & 30.78 / 0.839 & 26.56 / 0.807 & 29.47 / 0.818 & 24.12 / 0.705 & 26.53 / 0.707  & 27.44 / 0.683 & 26.71 / 0.734 \\ \hline
    IRCNN & {\bfseries 24.59} / 0.781	&30.19 / 0.835	&{26.94} / 0.813	&29.52 / 0.820	& {\bfseries 25.49} / {\bfseries 0.758}	&26.58 / 0.723	&27.55 / 0.706	&26.62 / 0.736 \\ \hline
    IDBP-BM3D & {24.51} / 0.775 & {\bfseries 31.14} / {\bfseries 0.844}  & {26.79} / 0.816  & {29.69} / 0.824 & {25.06} / 0.738 & {26.64} / 0.712  & {27.61} / 0.691  & {26.77} / 0.738 \\ \hline
    IDBP-CNN & 24.14 / {\bfseries 0.786}	&30.92 / 0.843	& {\bfseries 27.17} / {\bfseries 0.830} 	& {\bfseries 29.80} / {\bfseries 0.836}	&23.61 / 0.731	& {\bfseries 26.78} / {\bfseries 0.738}	& {\bfseries 27.70} / {\bfseries 0.714} &	{\bfseries 26.80} / {\bfseries 0.752}
 \\
    \hline
    \end{tabular}
\end{table*}

\begin{figure}
\captionsetup[subfigure]{labelformat=empty}
  \centering
  \begin{subfigure}[b]{0.5\linewidth}
    \centering\includegraphics[width=120pt]{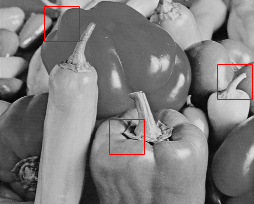}\\
\vspace{1mm}
    \centering\includegraphics[width=37.5pt]{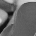}
    \centering\includegraphics[width=37.5pt]{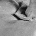}
    \centering\includegraphics[width=37.5pt]{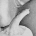}
    \caption{Original image}
    \vspace{1mm}
  \end{subfigure}%
  \begin{subfigure}[b]{0.5\linewidth}
    \centering\includegraphics[width=120pt]{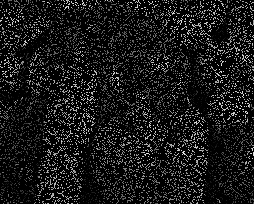}\\
\vspace{1mm}
    \centering\includegraphics[width=37.5pt]{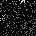}
    \centering\includegraphics[width=37.5pt]{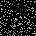}
    \centering\includegraphics[width=37.5pt]{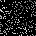}
    \caption{Subsampled and noisy image}
\vspace{1mm}
  \end{subfigure}
  \begin{subfigure}[b]{0.5\linewidth}
    \centering\includegraphics[width=120pt]{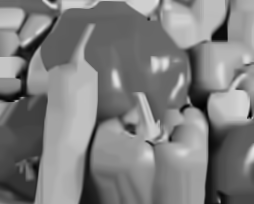}\\
\vspace{1mm}
    \centering\includegraphics[width=37.5pt]{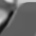}
    \centering\includegraphics[width=37.5pt]{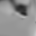}
    \centering\includegraphics[width=37.5pt]{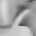}
    \caption{P\&P-BM3D (26.56 dB)}
\vspace{1mm}
  \end{subfigure}%
  \begin{subfigure}[b]{0.5\linewidth}
    \centering\includegraphics[width=120pt]{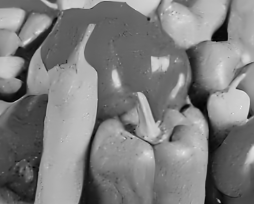}\\
\vspace{1mm}
    \centering\includegraphics[width=37.5pt]{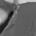}
    \centering\includegraphics[width=37.5pt]{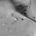}
    \centering\includegraphics[width=37.5pt]{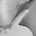}
    \caption{IRCNN (26.94 dB)}
\vspace{1mm}
  \end{subfigure}
  \begin{subfigure}[b]{0.5\linewidth}
    \centering\includegraphics[width=120pt]{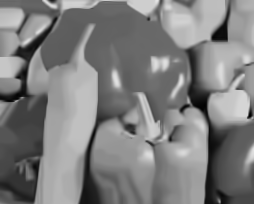}\\
\vspace{1mm}
    \centering\includegraphics[width=37.5pt]{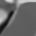}
    \centering\includegraphics[width=37.5pt]{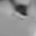}
    \centering\includegraphics[width=37.5pt]{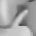}
    \caption{IDBP-BM3D (26.79 dB)}
  \end{subfigure}%
  \begin{subfigure}[b]{0.5\linewidth}
    \centering\includegraphics[width=120pt]{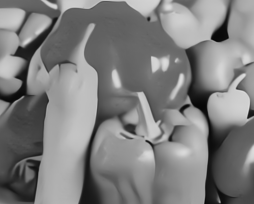}\\
\vspace{1mm}
    \centering\includegraphics[width=37.5pt]{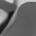}
    \centering\includegraphics[width=37.5pt]{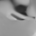}
    \centering\includegraphics[width=37.5pt]{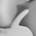}
    \caption{IDBP-CNN (27.17 dB)}
  \end{subfigure}
  \caption{Recovery of {\em peppers} image with 80\% missing pixels and $\sigma_e=12$. From left to right and from top to bottom: original image, subsampled and noisy image, reconstruction of P\&P-BM3D, reconstruction of IRCNN, reconstruction of the proposed IDBP-BM3D, and reconstruction of the proposed IDBP-CNN.}
\label{inpainting_example2}
\end{figure}

\begin{figure}
\captionsetup[subfigure]{labelformat=empty}
  \centering
  \begin{subfigure}[b]{0.5\linewidth}
    \centering\includegraphics[width=120pt]{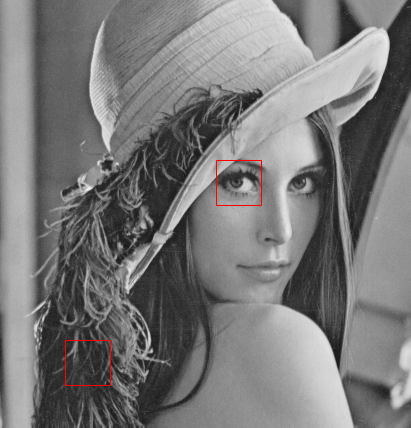}\\
\vspace{1mm}
    \centering\includegraphics[width=58pt]{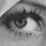}
    \centering\includegraphics[width=58pt]{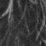}
    \caption{Original image}
\vspace{1mm}
  \end{subfigure}%
  \begin{subfigure}[b]{0.5\linewidth}
    \centering\includegraphics[width=120pt]{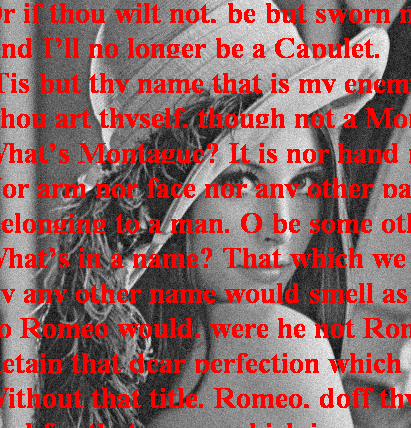}\\
\vspace{1mm}
    \centering\includegraphics[width=58pt]{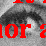}
    \centering\includegraphics[width=58pt]{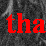}
    \caption{Degraded and noisy image}
\vspace{1mm}
  \end{subfigure}
  \begin{subfigure}[b]{0.5\linewidth}
    \centering\includegraphics[width=120pt]{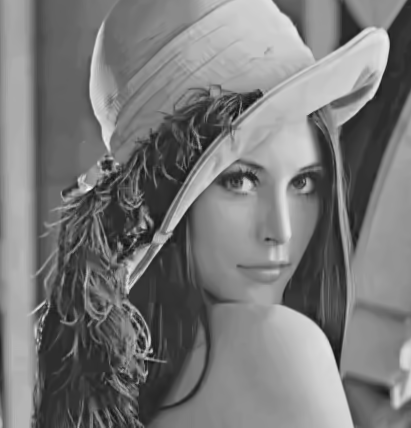}\\
\vspace{1mm}
    \centering\includegraphics[width=58pt]{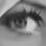}
    \centering\includegraphics[width=58pt]{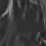}
    \caption{P\&P-BM3D (33.56 dB)}
\vspace{1mm}
  \end{subfigure}%
  \begin{subfigure}[b]{0.5\linewidth}
    \centering\includegraphics[width=120pt]{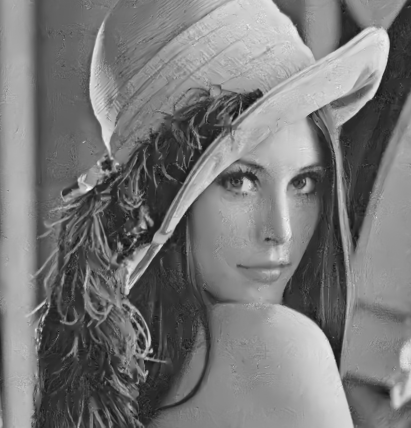}\\
\vspace{1mm}
    \centering\includegraphics[width=58pt]{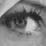}
    \centering\includegraphics[width=58pt]{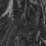}
    \caption{IRCNN (33.13 dB)}
\vspace{1mm}
  \end{subfigure}
  \begin{subfigure}[b]{0.5\linewidth}
    \centering\includegraphics[width=120pt]{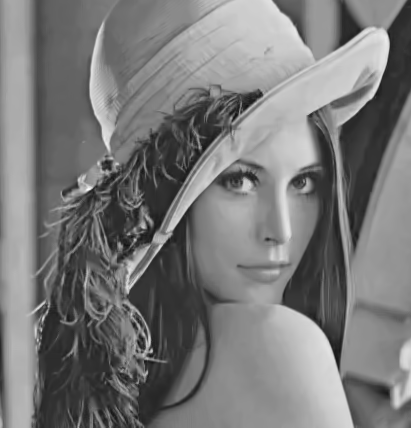}\\
\vspace{1mm}
    \centering\includegraphics[width=58pt]{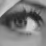}
    \centering\includegraphics[width=58pt]{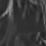}
    \caption{IDBP-BM3D (33.62 dB)}
  \end{subfigure}%
  \begin{subfigure}[b]{0.5\linewidth}
    \centering\includegraphics[width=120pt]{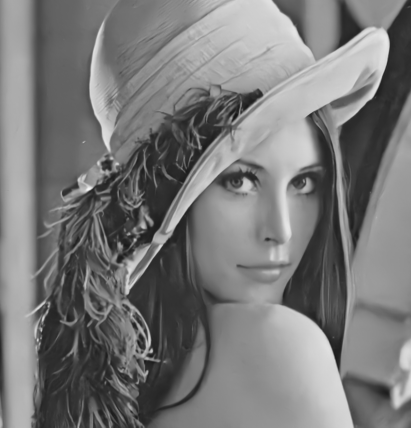}\\
\vspace{1mm}
    \centering\includegraphics[width=58pt]{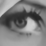}
    \centering\includegraphics[width=58pt]{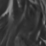}
    \caption{IDBP-CNN (33.94 dB)}
  \end{subfigure}
  \caption{Recovery of {\em Lena} image with superimposed text and $\sigma_e=10$. From left to right and from top to bottom: original image, degraded and noisy image, reconstruction of P\&P-BM3D, reconstruction of IRCNN, reconstruction of the proposed IDBP-BM3D, and reconstruction of the proposed IDBP-CNN.}
\label{inpainting_example3}
\end{figure}

\begin{table}
\renewcommand{\arraystretch}{1.3}
\caption{Average inpainting results (PSNR in dB / SSIM) for 50\% missing pixels on BSD68 dataset, and run-time (per image) on Intel i7-7500U CPU @ 2.70 GHz.} \label{table:bsd68_inpainting}
\centering
    \begin{tabular}{ | l | l | l | l | l |}
    \hline
            & $\sigma_e=0$ & $\sigma_e=10$ & $\sigma_e=12$ & Time \\ \hline
    P\&P-BM3D & 31.37 / 0.926 & 28.69 / 0.821 & 28.07 / 0.791 &  266s \\ \hline
    IRCNN & {\bfseries 31.91} / {\bfseries 0.929} & 29.05 / 0.838 & 28.51 / 0.817 & 34s  \\ \hline
    IDBP-BM3D & 31.18 / 0.918 & 28.78 / 0.823 & 28.29 / 0.802 & 179s \\ \hline
    IDBP-CNN & 31.21 / 0.915 & {\bfseries 29.26} / {\bfseries 0.846} & {\bfseries 28.77} / {\bfseries 0.824} & 33s \\ \hline
    \end{tabular}
\end{table}

\subsection{Image deblurring}
\label{exp_deblurring}

In the image deblurring problem, for a circular shift-invariant blur operator, 
both $\check{\x}_k$ in P\&P and $\tilde{\y}_k$ in IDBP can be efficiently implemented using Fast Fourier Transform (FFT). 
We use trivial initialization for all methods, i.e. $\check{\v}_0=\y$ in P\&P and $\tilde{\y}_0=\y$ in IDBP. The computational cost of each iteration of P\&P-BM3D and IDBP-BM3D is of the same scale, dominated by the complexity of the denoising operation.\footnote{Yet, since $\tilde{\y}_0=\y$ and $\mathcal{D}(\y;\sigma) \approx \y$, the denoising operation in IDBP first iteration can even be spared and replaced by $\tilde{\x}_1 = \y$.}  
Therefore, similarly to inpainting, the overall complexity of P\&P-BM3D and IDBP-BM3D is determined by the number of iterations of each strategy. 
The same goes for IRCNN and IDBP-CNN.

We consider four deblurring scenarios used as benchmarks in many publications (e.g. \cite{guerrero2008image, danielyan2012bm3d}). The blur kernel $h(x_1,x_2)$ and noise level of each scenario are summarized in Table \ref{table:blur_kernels}. The kernels are normalized such that $\sum_{x_1,x_2} h(x_1,x_2) = 1$.

\begin{table}
\renewcommand{\arraystretch}{1.3}
\caption{Blur kernel and noise variance of different scenarios.} \label{table:blur_kernels}
\centering
    \begin{tabular}{ | l | l | l |}
    \hline
    Scenario   & $h(x_1,x_2)$ & $\sigma_e^2$  \\ \hline
    1 & $1/(x_1^2+x_2^2),\,\, x_1, x_2 = -7,\ldots,7$ & 2  \\ \hline
    2 & $1/(x_1^2+x_2^2),\,\, x_1, x_2 = -7,\ldots,7$ & 8  \\ \hline
    3 & $9 \times 9$ uniform & $\approx$ 0.3  \\ \hline
    4 & $[1,4,6,4,1]^T[1,4,6,4,1]/256$ & 49  \\ \hline
    \end{tabular}
\end{table}

We compare the performance of IDBP, P\&P and IRCNN with IDD-BM3D \cite{danielyan2012bm3d}, which is a state-of-the-art algorithm specifically designed for deblurring. We use IDD-BM3D exactly as in \cite{danielyan2012bm3d}\footnote{Downloaded from \url{http://www.cs.tut.fi/~foi/GCF-BM3D/BM3D.zip}.}, where the same scenarios are examined: it is initialized using BM3D-DEB \cite{dabov2008image}, performs 200 iterations and its parameters are manually tuned per scenario. The parameters of P\&P-BM3D are also optimized for each scenario. It uses 50 iterations and $\beta=\{$0.85, 0.85, 0.9, 0.8$\}$ and $\lambda=\{$2, 1, 3, 1$\}/255$, for scenarios 1-4, respectively.
For IRCNN we use the exact code supplied by the authors, which performs 30 iterations and does not use the DNNs which are associated with low noise levels. 
Our efforts to use also the low noise level DNNs in IRCNN have not been successful.

We examine two different tuning strategies for IDBP. The first one is a manual tuning per scenario, which is simpler than the tuning of the competing methods. For IDBP-BM3D, we fix $\delta=5$ for all scenarios and only change $\epsilon$ to $\{$7e-3, 4e-3, 8e-3, 2e-3$\}$ for scenarios 1-4, respectively. For IDBP-CNN, we fix $\delta=10$ \footnote{We increase $\delta$ to avoid using DNNs associated with low noise levels, similar to IRCNN, as such DNNs lead to inferior deblurring results.} for all scenarios and only change $\epsilon$ to $\{$4e-3, 2e-3, 3e-3, 0.8e-3$\}$ for scenarios 1-4, respectively.
The second strategy applies Algorithm \ref{Auto_IDBP_alg} with the suggested default settings for IDBP-BM3D, and with a minor change of setting, i.e. $\delta=10$ and $\tau=4$, for IDBP-CNN. In both cases, the multiplication $\epsilon \cdot \sigma_e^2$ in (\ref{Eq_Hinv_approx}) is kept above 5e-4, to prevent numerical complications for $\sigma_e \ll 1$. 
Note that in the automatic-tuning scheme, $\epsilon$ can be set differently for different images in the same scenario, while all images, in all scenarios, use the same method with the same default parameters. 
We use a stopping criterion of only 30 iterations for all IDBP versions (which is equal to the number of iterations of IRCNN and lower than P\&P).
 
Table \ref{table:deblurring_S1to4} shows the results of the different methods. For each scenario it shows the input PSNR (i.e. PSNR of $\y$) and the BSNR (blurred signal-to-noise-ratio, defined as $\mathrm{var}(\H\x)/m\sigma_e^2$) for each image, as well as the ISNR\footnote{ISNR equals the difference between the PSNR of the reconstruction and the input PSNR.} (improvement signal-to-noise-ratio) and SSIM for each method and image. Note that in Scenario 3, $\sigma_e^2$ is set slightly different for each image, ensuring that the BSNR is 40 dB. 
In Fig. \ref{fig:deblurring_PSNR_vs_iter} we also present the PSNR, averaged over all scenarios and images, as a function of the iteration number for the denoising-based techniques (i.e. P\&P, IRCNN and IDBP). It can be seen that the IDBP optimization scheme is faster than P\&P (which is based on variable splitting and ADMM) and IRCNN (which is based on variable splitting and quadratic penalty method).

From Table \ref{table:deblurring_S1to4} it is clear that IDBP's plain and auto-tuned implementations have similar performance on average. Moreover, note that in every scenario at least one of the two leading methods (on average) is obtained by an IDBP implementation. In fact, when averaging over all scenarios IDBP-CNN has the highest ISNR.

\begin{figure}
    \centering
    \includegraphics[width=0.39\textwidth]{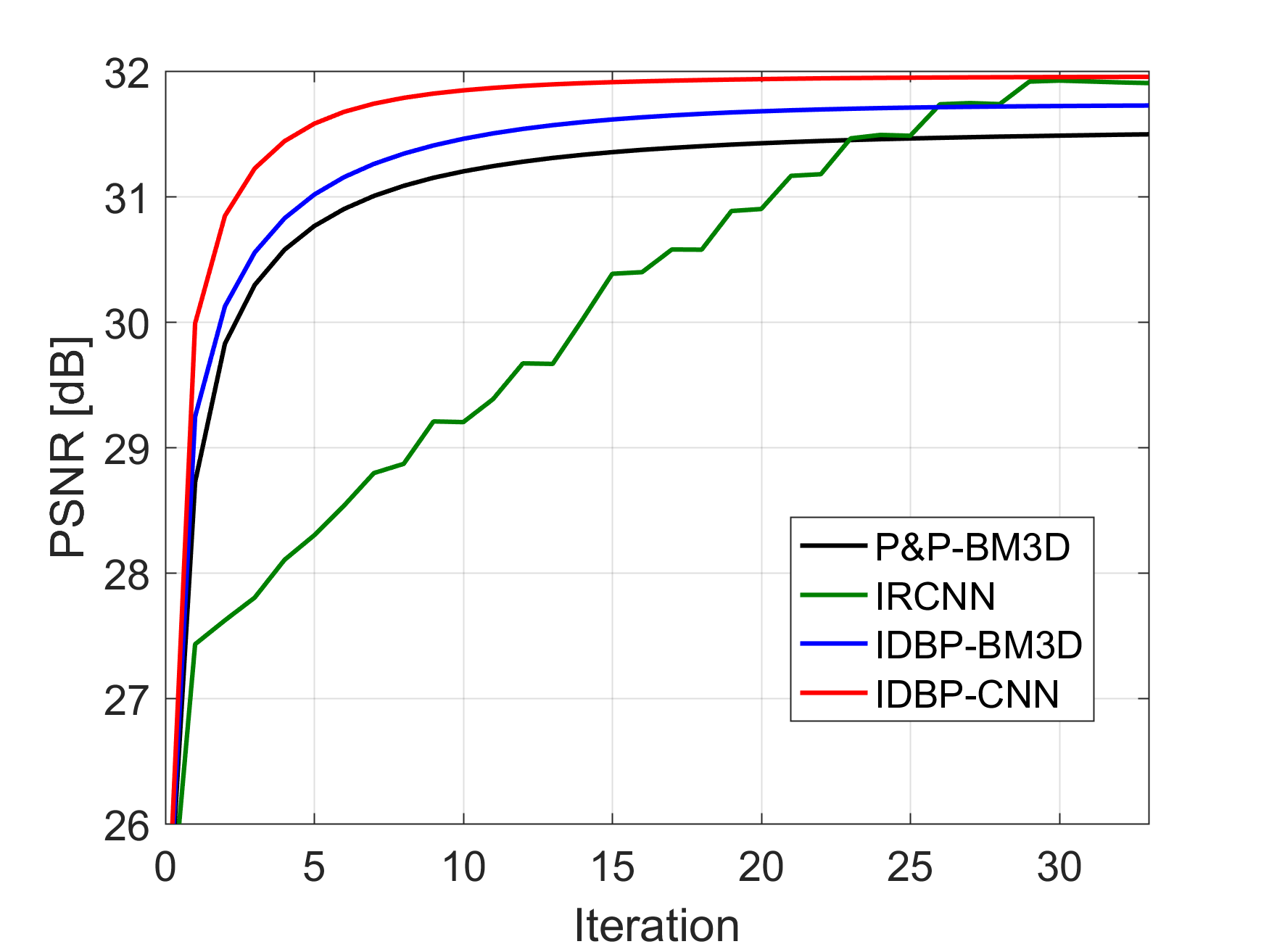}
    \caption{Deblurring results (PSNR averaged over all scenarios and images in Table \ref{table:deblurring_S1to4} vs. iteration number) for different denoising-based methods.} \label{fig:deblurring_PSNR_vs_iter}
\end{figure}

Focusing on the methods that use the BM3D prior, both IDBP-BM3D versions perform better than P\&P-BM3D, and have only a small performance gap below IDD-BM3D, which is especially tailored for the deblurring problem and requires many more iterations and parameter tuning.
Fig. \ref{deblurring_example} displays the results of the BM3D-based methods for {\em Barbara} in Scenario 4.
It can be seen that IDBP-BM3D reconstruction, especially with auto-tuning in this case, restores the texture better (e.g. look near the right palm).

Focusing on the methods that use DNN prior, these three methods are highly competitive. However, we remind the reader that IRCNN uses about two dozen different DNNs for each inverse problem, while IDBP requires only a single DNN per scenario, as we do not modify $\delta$ between iterations.
Fig. \ref{deblurring_example2} displays the results of IRCNN and IDBP-CNN for {\em Lena} in Scenario 1.
It can be seen that IDBP-CNN recovers more fine details of the hat.
Fig. \ref{deblurring_example3} displays the results of IRCNN and IDBP-CNN for {\em Cameraman} in Scenario 3.
In this case, IDBP-CNN recovers the fingers better. 
We do not show the results of auto-tuned IDBP-CNN, since there is no visual difference compared to its plain counterpart.

We examine the performance of the methods on BSD68 dataset as well.
The results are given in Table \ref{table:bsd68_deblurring}. The observations in the above three paragraphs stay the same on this dataset. IDBP's plain and auto-tuned implementations exhibit (similar) impressive performance. Again, when averaging over all scenarios IDBP-CNN has the highest PSNR. 

For more examples, we refer the reader to a short conference version of this paper \cite{tirer2018icip}, where we demonstrate the advantages of using IDBP with an automatic parameter tuning for image deblurring when only an inexact estimate of the blur kernel is available, and per-scenario tuning is impossible.
This situation is encountered in most blind-deblurring methods, which start with estimating only the kernel, and then use it to recover the latent image via non-blind deblurring.

\begin{table*}[t]
\scriptsize
\renewcommand{\arraystretch}{1.3}
\caption{Deblurring inputs (BSNR and input PSNR in dB) and reconstruction results (Improvement SNR in dB / SSIM for each method) for scenarios 1-4.}
\label{table:deblurring_S1to4}
\centering
    \begin{tabular}{|p{2.503cm}|p{1.29cm}|p{1.29cm}|p{1.29cm}|p{1.21cm}|p{1.21cm}|p{1.21cm}|p{1.21cm}|p{1.29cm}|p{1.21cm}|}
    \hline
    Scenario 1  & {\em cameraman} & {\em house} & {\em peppers} & {\em Lena} & {\em Barbara} & {\em boat} & {\em hill} & {\em couple} & Average \\ \hline
    BSNR & 31.87 & 29.16  & 29.99 & 29.89 & 30.81 & 29.37 & 30.19 & 28.81 &\\ \hline
    input PSNR & 22.23 & 25.61  & 22.60 & 27.25 & 23.34 & 25.00 & 26.51 & 24.87 & \\ \hline
    \hline
    IDD-BM3D & {8.86} / 0.886 & {\bfseries 9.95} / 0.891  & {10.46} / 0.918 & 7.97 / 0.902 & 7.64 / 0.897 & {7.68} / 0.870 & {6.03} / {\bfseries 0.859} & {7.61} / 0.889 & 8.28 / 0.889 \\ \hline
    P\&P-BM3D & 8.03 / 0.883 & 9.74 / 0.890  & 10.02 / 0.921 & {8.02} / 0.909 & 6.84 / 0.890 & 7.48 / 0.870 & 5.78 / 0.855 & 7.34 / 0.886 & 7.91 / 0.888 \\ \hline
    IRCNN & {\bfseries 9.08} / 0.894 	&9.69 / 0.884	&10.83 / 0.923	&8.06 / 0.906	&7.54 / 0.898	&7.71 / 0.867	&5.88 / 0.847	&7.64 / 0.882  & 8.30 / 0.888 \\ \hline
    IDBP-BM3D & 8.51 / 0.893 & 9.82 / 0.891  & 10.07 / 0.920  & 7.92 / 0.909 & {\bfseries 7.90} / {\bfseries 0.906} & 7.54 / 0.871 & 5.90 / 0.854 & 7.34 / 0.885 & 8.13 / 0.891 \\ \hline 
    Auto-tuned IDBP-BM3D & 8.40 / 0.890 & 9.83 / 0.890  & 10.06 / 0.920  & {8.02} / 0.910 & 7.59 / 0.901 & 7.61 / 0.870 & 5.90 / 0.852 & 7.46 / 0.885 & 8.11 / 0.890 \\  \hline 
    IDBP-CNN & {\bfseries 9.08} / {\bfseries 0.897} 	&9.93 / {\bfseries 0.892}	& {\bfseries 10.97} / {\bfseries 0.926}	&8.24 / {\bfseries 0.911}	&6.89 / 0.895	& {\bfseries 7.81} / {\bfseries 0.873}	& {\bfseries 6.04} / {\bfseries 0.859}	& {\bfseries 7.75} / {\bfseries 0.890} & {\bfseries 8.34} / {\bfseries 0.893} \\ \hline 
    Auto-tuned IDBP-CNN & 9.07 / {\bfseries 0.897}	&9.92 / {\bfseries 0.892}	& {\bfseries 10.97} / {\bfseries 0.926}	& {\bfseries 8.25} / {\bfseries 0.911}	&6.84 / 0.894	& {\bfseries 7.81} / {\bfseries 0.873}	& {\bfseries 6.04} / {\bfseries 0.859}	& {\bfseries 7.75} / {\bfseries 0.890} & 8.33 / {\bfseries 0.893} \\  \hline %
    \hline
    \end{tabular}
\centering
    \begin{tabular}{|p{2.503cm}|p{1.29cm}|p{1.29cm}|p{1.29cm}|p{1.21cm}|p{1.21cm}|p{1.21cm}|p{1.21cm}|p{1.29cm}|p{1.21cm}|}
    \hline
    Scenario 2  & {\em cameraman} & {\em house} & {\em peppers} & {\em Lena} & {\em Barbara} & {\em boat} & {\em hill} & {\em couple} & Average  \\ \hline
    BSNR & 25.85 & 23.14  & 23.97 & 23.87 & 24.79 & 23.35 & 24.17 & 22.79 &\\ \hline
    input PSNR & 22.16 & 25.46  & 22.53 & 27.04 & 23.25 & 24.88 & 26.33 & 24.75 &\\ \hline
    \hline
    IDD-BM3D & {7.12} / 0.856 & {8.55} / 0.872  & {8.65} / 0.894 & {6.61} / 0.881 & {3.96} / 0.822 & {5.96} / 0.832 &  {\bfseries 4.69} / {\bfseries 0.813} & {5.88} / 0.847 & 6.43 / 0.852 \\ \hline
    P\&P-BM3D & 6.06 / 0.842 & 8.20 / 0.866 & 8.15 / 0.894 & 6.49 / 0.883 & 2.72 / 0.788 & 5.65 / 0.828 & 4.46 / 0.809 & 5.56 / 0.841 & 5.91 / 0.844 \\ \hline 
    IRCNN & {\bfseries 7.33} / {\bfseries 0.867}	& {\bfseries 8.63} / {\bfseries 0.876}	& {\bfseries 9.09} / {\bfseries 0.901}	& {\bfseries 6.79} / {\bfseries 0.888}	& {\bfseries 4.68} / {\bfseries 0.841}	& {\bfseries 6.11} / {\bfseries 0.836}	&4.62 / {\bfseries 0.813}	& {\bfseries 6.06} / {\bfseries 0.851} & {\bfseries 6.66} / {\bfseries 0.859} \\ \hline
    IDBP-BM3D & 6.61 / 0.858 & 8.15 / 0.863  & 7.97 / 0.890  & 6.58 / {\bfseries 0.888} & 3.94 / 0.830 & 5.87 / 0.835 & 4.61 / 0.812 & 5.71 / 0.846 & 6.18 / 0.853 \\ \hline 
    Auto-tuned IDBP-BM3D & 6.56 / 0.858 & 8.15 / 0.863  & 8.00 / 0.892  & 6.54 / 0.887 & 3.94 / 0.830 & 5.91 / 0.835 & 4.61 / 0.812 & 5.77 / 0.846 & 6.19 / 0.853 \\  \hline 
    IDBP-CNN & 7.28 / 0.866	&8.45 / 0.874	&8.96 / 0.898	&6.64 / 0.886	&4.41 / 0.838	&5.97 / 0.832	&4.43 / 0.808	&5.91 / 0.848 & 6.51 / 0.856 \\ \hline 
    Auto-tuned IDBP-CNN & 7.23 / 0.864 	&8.53 / 0.874	&9.04 / {\bfseries 0.901}	&6.71 / 0.887	&4.01 / 0.829	&6.04 / 0.834	&4.52 / 0.810	&5.97 / 0.849 & 6.51 / 0.856 \\  \hline %
    \hline
    \end{tabular}
\centering
    \begin{tabular}{|p{2.503cm}|p{1.29cm}|p{1.29cm}|p{1.29cm}|p{1.21cm}|p{1.21cm}|p{1.21cm}|p{1.21cm}|p{1.29cm}|p{1.21cm}|}
    \hline
    Scenario 3  & {\em camera.} & {\em house} & {\em peppers} & {\em Lena} & {\em Barbara} & {\em boat} & {\em hill} & {\em couple} & Average  \\ \hline
    BSNR & 40.00 & 40.00  & 40.00 & 40.00 & 40.00 & 40.00 & 40.00 & 40.00 &\\ \hline
    input PSNR & 20.77 & 24.11  & 21.33 & 25.84 & 22.49 & 23.36 & 25.04 & 23.24 &\\ \hline
    \hline
    IDD-BM3D & {10.45} / 0.895 & 12.89 / 0.920  & {12.06} / 0.922 & 8.91 / 0.900 & 6.05 / 0.847 & {\bfseries 9.77} / 0.880 &  {\bfseries 7.78} / {\bfseries 0.868} & {\bfseries 10.06} / 0.906 & {\bfseries 9.75} / 0.892 \\ \hline
    P\&P-BM3D & 9.49 / 0.894 & {\bfseries 13.17} / {\bfseries 0.930}  & 11.70 / {\bfseries 0.926} & 9.04 / {\bfseries 0.908} & 5.36 / 0.830 & 9.71 / {\bfseries 0.883} & 7.63 / 0.867 & 9.98 / {\bfseries 0.909} & 9.51 / 0.893 \\ \hline
    IRCNN & 10.30 / 0.887	&11.58 / 0.886	&12.03 / 0.920	&8.88 / 0.899	&5.92 / 0.841	&9.36 / 0.864	&7.22 / 0.841	&9.48 / 0.883 & 9.35 / 0.878 \\ \hline
    IDBP-BM3D & 9.78 / {\bfseries 0.898} & 12.96 / 0.928  & 11.92 / 0.925  & 9.03 / 0.906 & {6.22} / 0.855 & 9.64 / 0.880 & 7.66 / 0.863 & 9.85 / 0.905 & 9.63 / {\bfseries 0.895} \\ \hline 
    Auto-tuned IDBP-BM3D & 9.67 / 0.895 & 12.96 / 0.927  & 11.90 / 0.925  & {\bfseries 9.07} / 0.906 & 6.01 / 0.848 & 9.74 / 0.879 & 7.67 / 0.862 & 9.98 / 0.904 & 9.63 / 0.893 \\  \hline 
    IDBP-CNN & {\bfseries 10.55} / 0.896	&11.91 / 0.894	&{\bfseries 12.33} / 0.925	 &9.05 / 0.904	&6.07 / 0.856	&9.63 / 0.874	&7.49 / 0.856	&9.91 / 0.897 & 9.62 / 0.888  \\ \hline 
    Auto-tuned IDBP-CNN & 10.54 / 0.896	&11.91 / 0.895	&12.30 / 0.924	&8.99 / 0.903	& {\bfseries 6.26} / {\bfseries 0.861}	&9.62 / 0.874	&7.37 / 0.853	&9.93 / 0.898 & 9.62 / 0.888 \\  \hline %
    \hline
    \end{tabular}
\centering
    \begin{tabular}{|p{2.503cm}|p{1.29cm}|p{1.29cm}|p{1.29cm}|p{1.21cm}|p{1.21cm}|p{1.21cm}|p{1.21cm}|p{1.29cm}|p{1.21cm}|}
    \hline
    Scenario 4  & {\em cameraman} & {\em house} & {\em peppers} & {\em Lena} & {\em Barbara} & {\em boat} & {\em hill} & {\em couple} & Average  \\ \hline
    BSNR & 18.53 & 15.99  & 17.01 & 16.47 & 17.35 & 16.06 & 16.68 & 15.55 &\\ \hline
    input PSNR & 24.62 & 28.06  & 24.77 & 28.81 & 24.22 & 27.10 & 27.74 & 26.94 &\\ \hline
    \hline
    IDD-BM3D & {3.98} / 0.853 & {5.79} / 0.870  & 4.45 / 0.879 & 4.97 / 0.883 & 1.88 / 0.801 & {3.60} / 0.836 &  {\bfseries 3.29} / {\bfseries 0.818} & {3.61} / 0.849 & 3.95 / 0.849 \\ \hline
    P\&P-BM3D & 3.31 / 0.842 & 5.43 / 0.863  & {4.95} / 0.887 & 4.84 / 0.884 & 1.50 / 0.787 & 3.42 / 0.833 & 3.13 / 0.817 & 3.39 / 0.845 & 3.75 / 0.845 \\ \hline
    IRCNN & {\bfseries 4.29} /  {\bfseries 0.862}	& {\bfseries 6.05} / {\bfseries 0.875}	& {\bfseries 6.66} / {\bfseries 0.902}	&{\bfseries 5.13} / 0.889	&1.82 / 0.802	& {\bfseries 3.84} / {\bfseries 0.838}	&3.26 / {\bfseries 0.818}	& {\bfseries 3.74} / {\bfseries 0.850} & {\bfseries 4.35} / {\bfseries 0.855}  \\ \hline
    IDBP-BM3D & 3.61 / 0.854 & 5.69 / 0.871  & 4.44 / 0.884  & {5.07} / {\bfseries 0.891} & 1.97 / 0.809 & 3.54 / 0.834 & 3.12 / 0.809 & 3.50 / 0.845 & 3.87 / 0.850 \\ \hline 
    Auto-tuned IDBP-BM3D & 3.65 / 0.848 & 5.42 / 0.863  & 4.36 / 0.877  & 4.94 / 0.888 & {\bfseries 2.72} / {\bfseries 0.830}  & 3.52 / 0.834 & 3.15 / 0.811 & 3.41 / 0.844 & 3.90 / 0.849 \\  \hline 
    IDBP-CNN & 4.25 / 0.860	&5.85 / 0.871	&6.29 / 0.900	&5.05 / 0.889	&2.40 / 0.819	&3.68 / 0.831	&3.19 / 0.808	&3.63 / 0.841 & 4.29 / 0.852  \\ \hline 
    Auto-tuned IDBP-CNN & 4.20 / 0.858	&5.85 / 0.870	&6.28 / 0.899	&5.04 / 0.888	&2.31 / 0.816	&3.67 / 0.829	&3.17 / 0.805	&3.59 / 0.839 & 4.26 / 0.851  \\  \hline %
    \end{tabular}
\end{table*}

\begin{figure}
\captionsetup[subfigure]{labelformat=empty}
  \centering
  \begin{subfigure}[b]{1\linewidth}
    \centering\includegraphics[width=150pt]{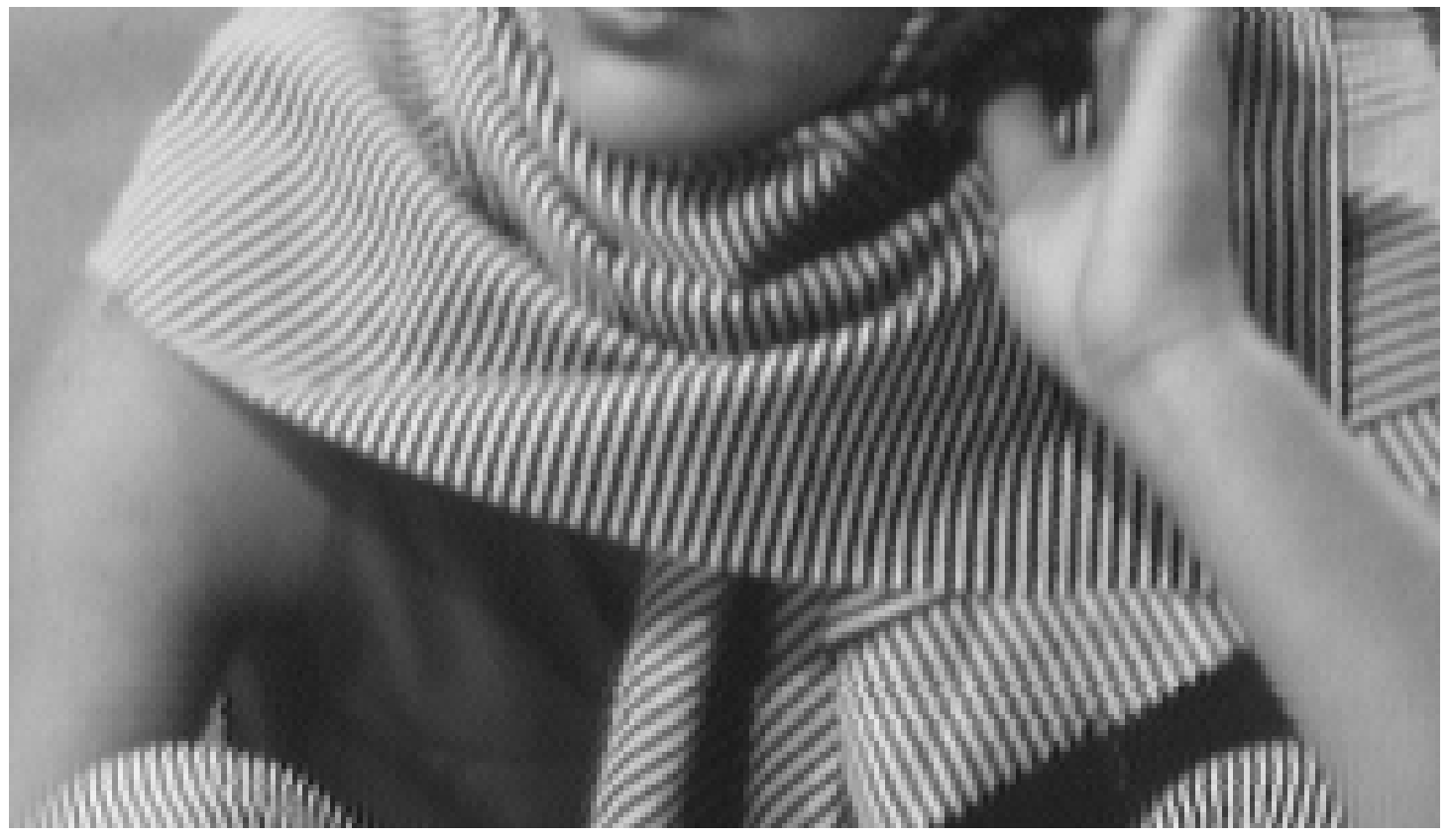}
    \centering\includegraphics[width=88pt]{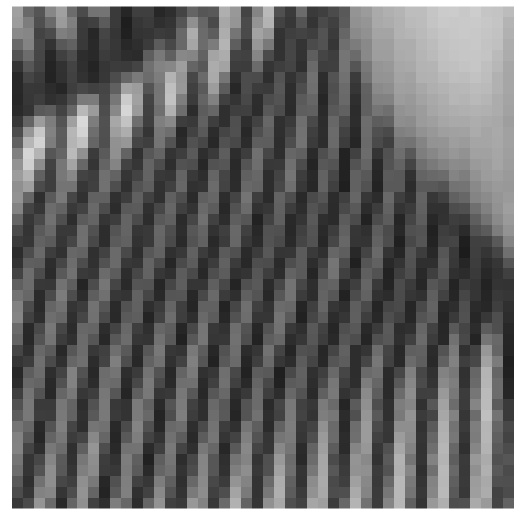}\\
    \caption{Original image}
    \vspace{1mm}
  \end{subfigure}
\\
  \begin{subfigure}[b]{1\linewidth}
    \centering\includegraphics[width=150pt]{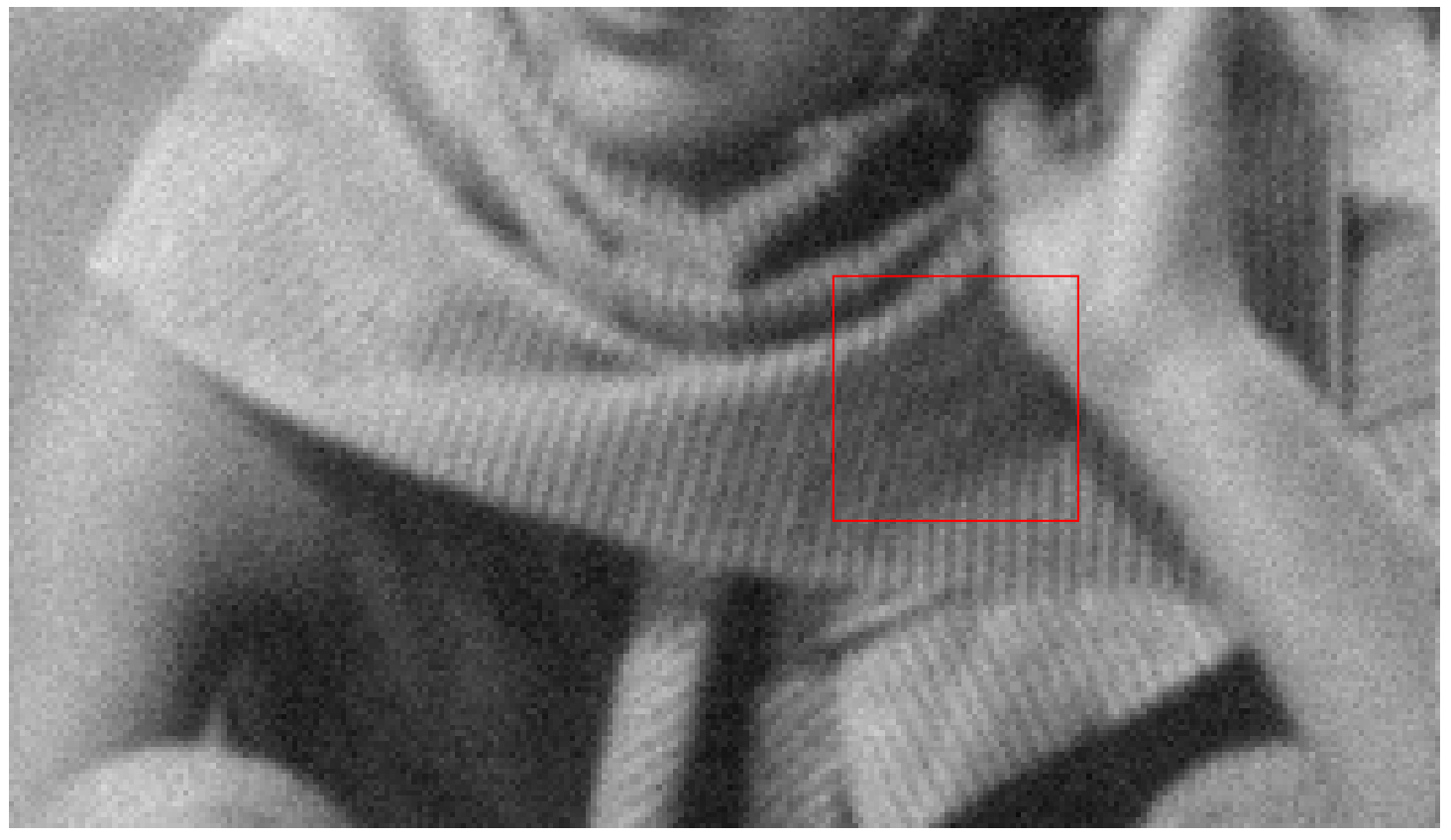}
    \centering\includegraphics[width=88pt]{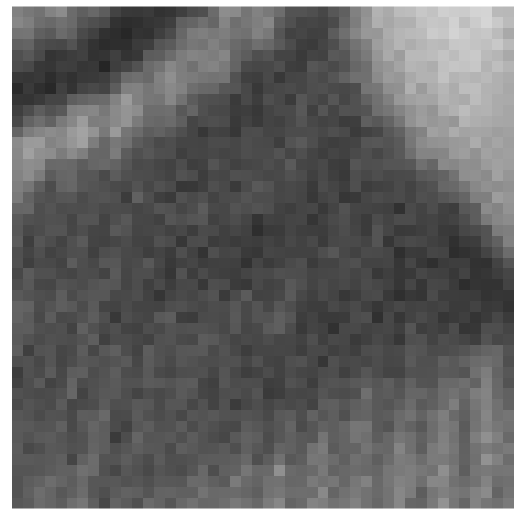}\\
    \caption{Blurred and noisy image}
    \vspace{1mm}
  \end{subfigure}
\\
  \begin{subfigure}[b]{1\linewidth}
    \centering\includegraphics[width=150pt]{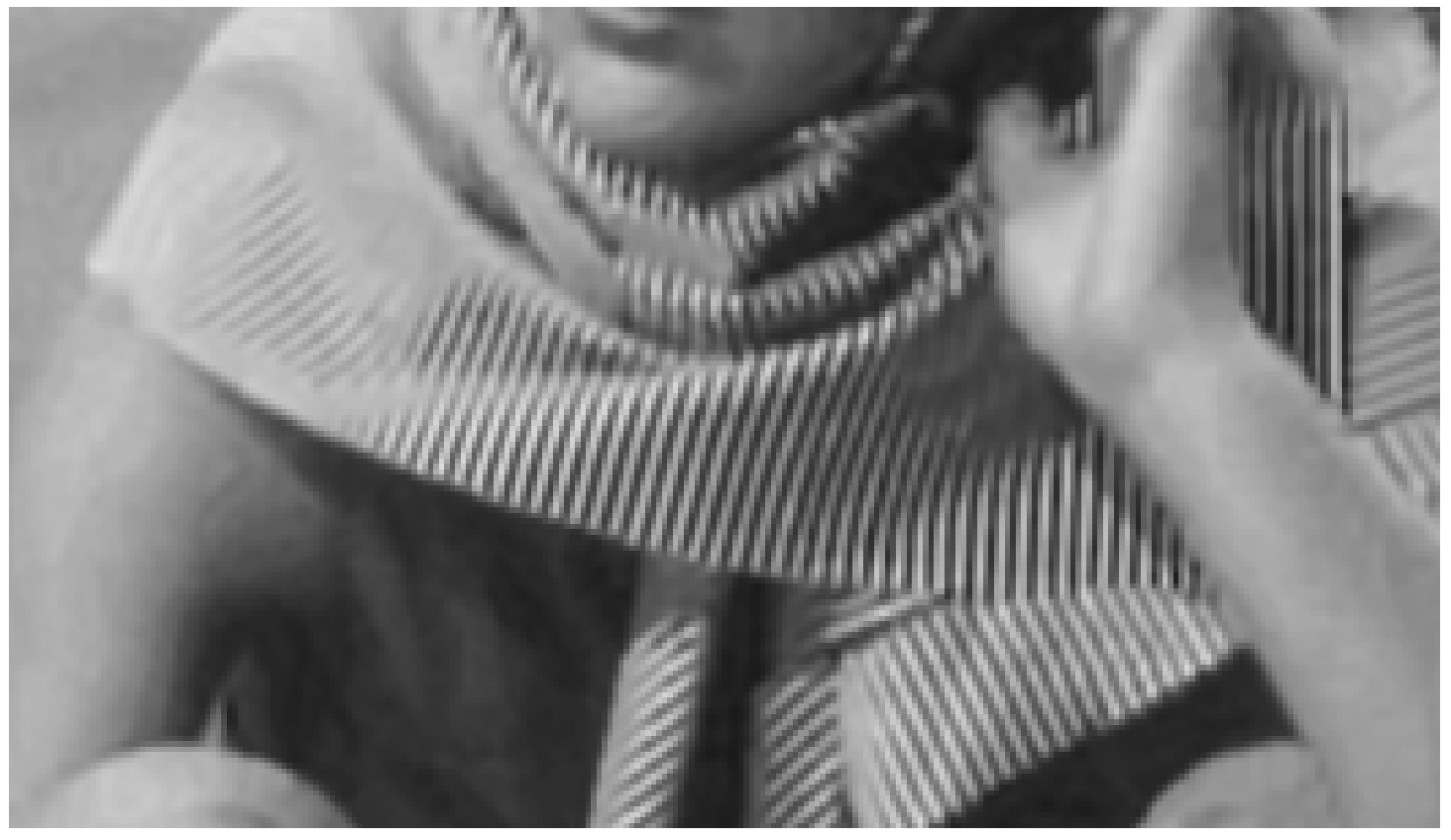}
    \centering\includegraphics[width=88pt]{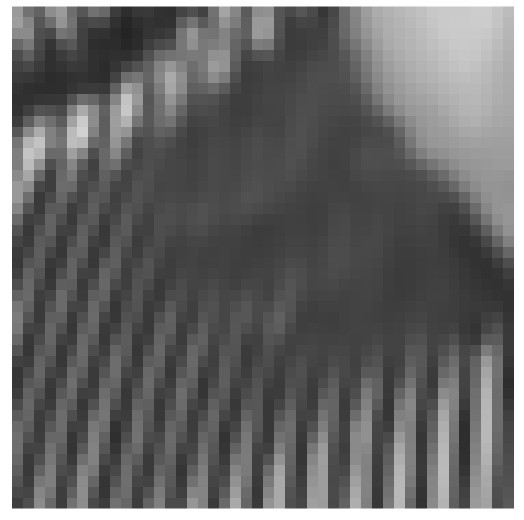}\\
    \caption{IDD-BM3D (26.10 dB)}
    \vspace{1mm}
  \end{subfigure}
\\
  \begin{subfigure}[b]{1\linewidth}
    \centering\includegraphics[width=150pt]{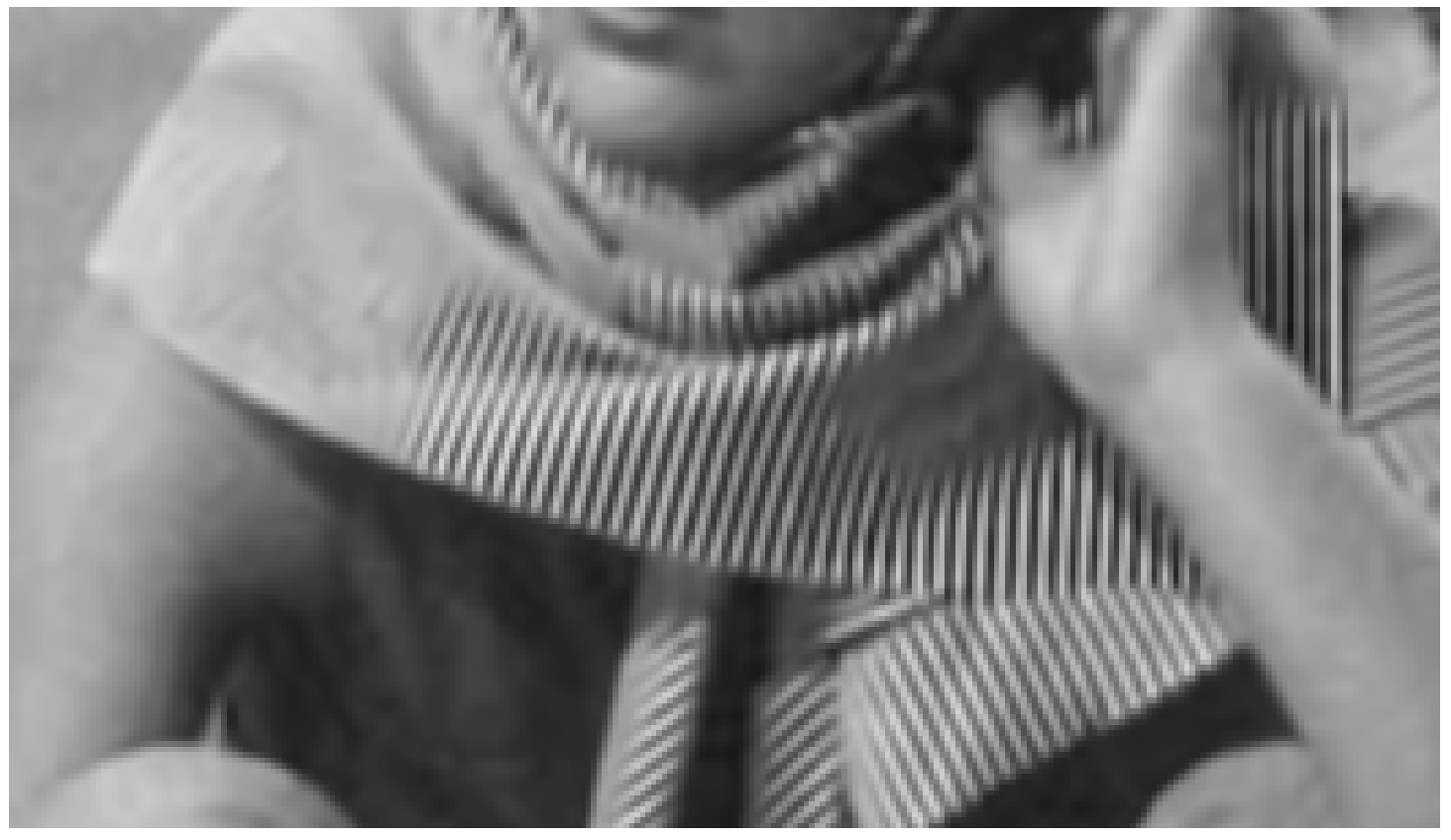}
    \centering\includegraphics[width=88pt]{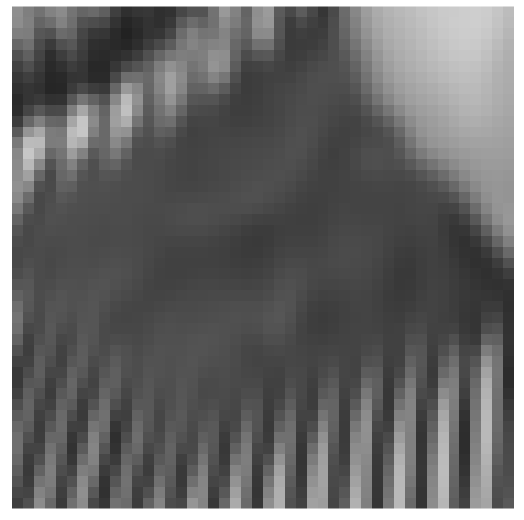}\\
    \caption{P\&P-BM3D (25.72 dB)}
    \vspace{1mm}
  \end{subfigure}
\\
  \begin{subfigure}[b]{1\linewidth}
    \centering\includegraphics[width=150pt]{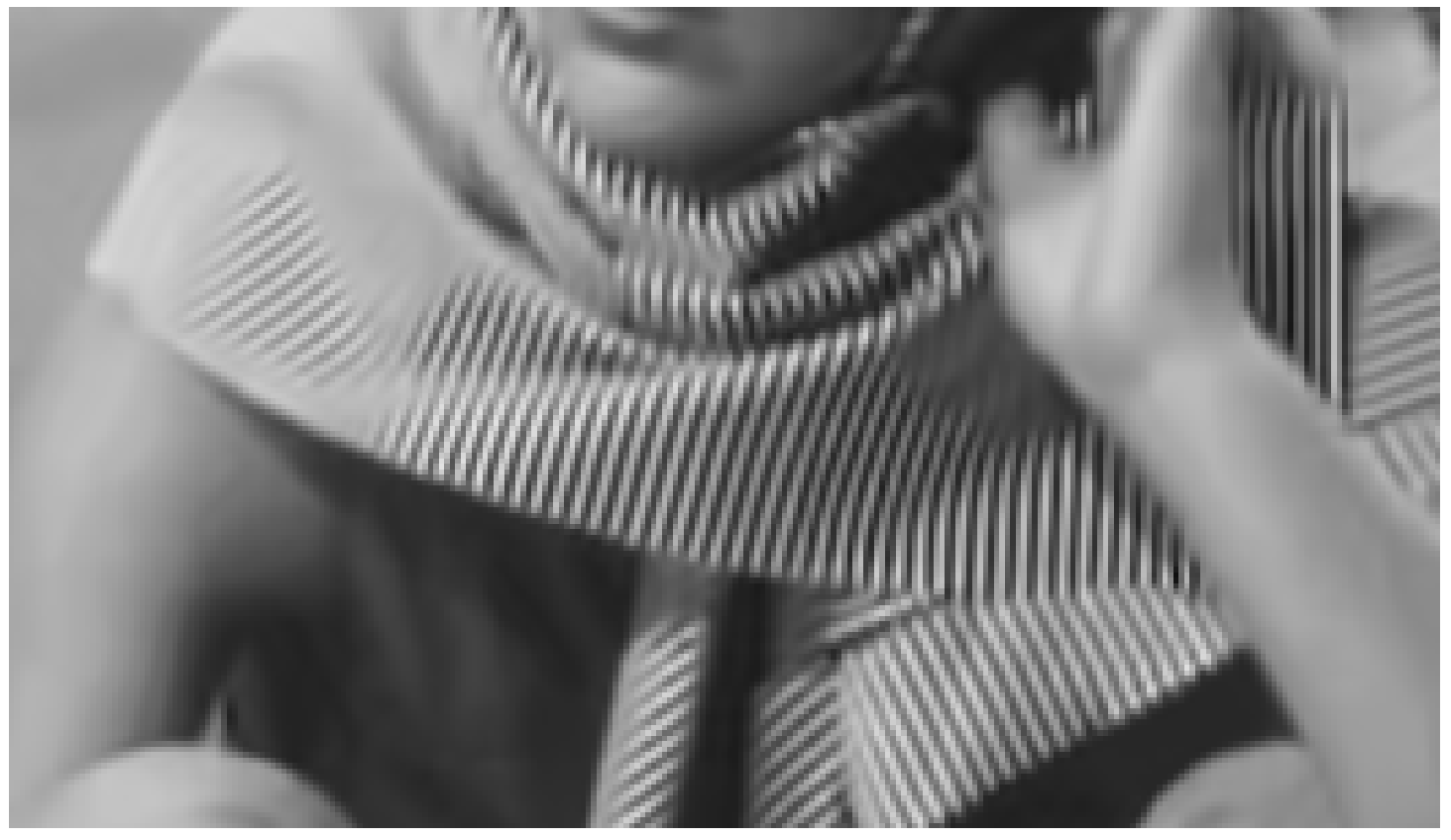}
    \centering\includegraphics[width=88pt]{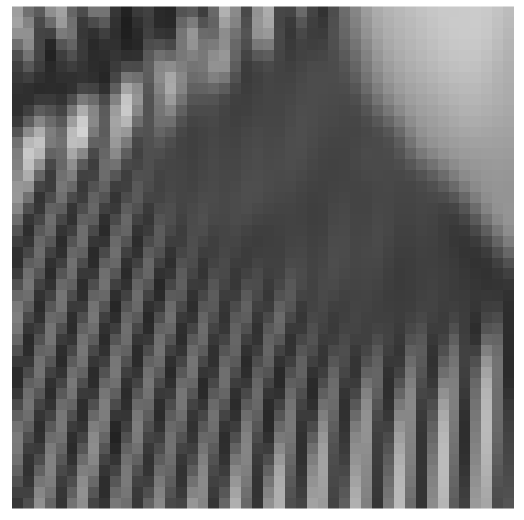}\\
    \caption{IDBP-BM3D (26.19 dB)}
    \vspace{1mm}
  \end{subfigure}
\\
  \begin{subfigure}[b]{1\linewidth}
    \centering\includegraphics[width=150pt]{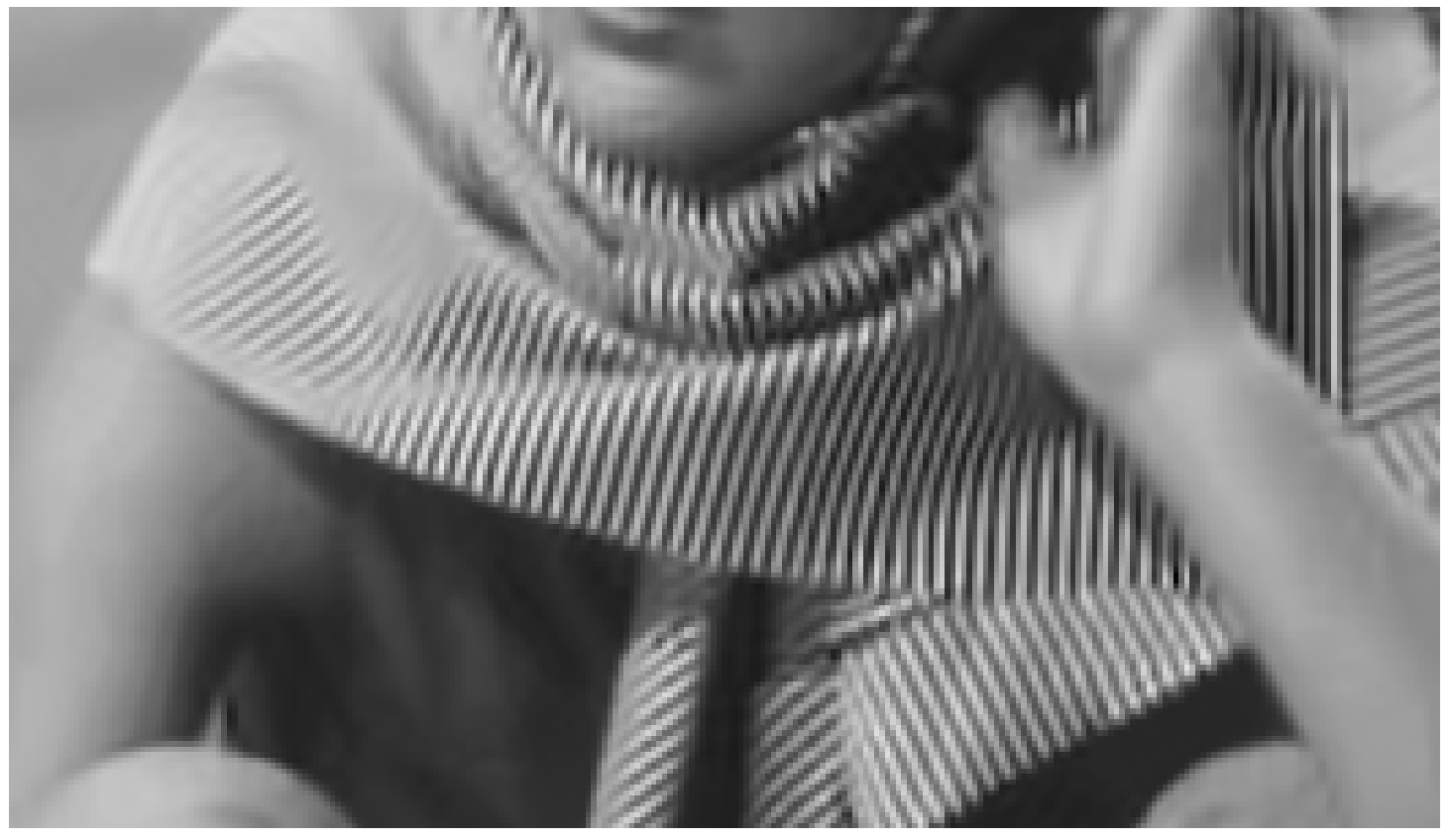}
    \centering\includegraphics[width=88pt]{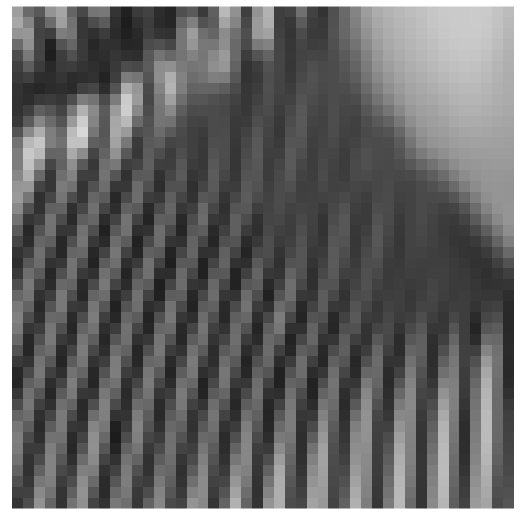}\\
    \caption{Auto-tuned IDBP-BM3D (26.94 dB)}
    \vspace{1mm}
  \end{subfigure}  
  \caption{Deblurring of {\em Barbara} image, Scenario 4. From top to bottom, fragments of: original image, blurred and noisy image, reconstruction of IDD-BM3D, reconstruction of P\&P-BM3D, reconstruction of the proposed IDBP-BM3D, and reconstruction of the proposed auto-tuned IDBP-BM3D.}
\label{deblurring_example}
\end{figure}

\begin{figure}
\captionsetup[subfigure]{labelformat=empty}
  \centering
  \begin{subfigure}[b]{1\linewidth}
    \centering\includegraphics[width=150pt]{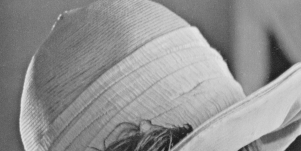}
    \centering\includegraphics[width=75.3pt]{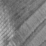}\\
    \caption{Original image}
    \vspace{1mm}
  \end{subfigure}
\\
  \begin{subfigure}[b]{1\linewidth}
    \centering\includegraphics[width=150pt]{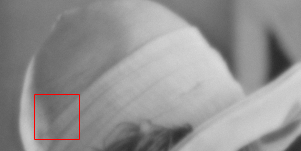}
    \centering\includegraphics[width=75.3pt]{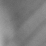}\\
    \caption{Blurred and noisy image}
    \vspace{1mm}
  \end{subfigure}
\\
  \begin{subfigure}[b]{1\linewidth}
    \centering\includegraphics[width=150pt]{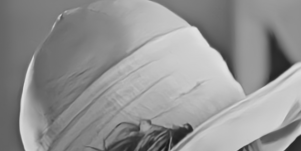}
    \centering\includegraphics[width=75.3pt]{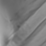}\\
    \caption{IRCNN (35.31 dB)}
    \vspace{1mm}
  \end{subfigure}
\\
  \begin{subfigure}[b]{1\linewidth}
    \centering\includegraphics[width=150pt]{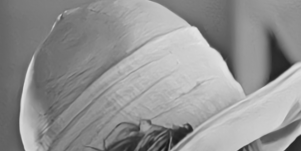}
    \centering\includegraphics[width=75.3pt]{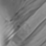}\\
    \caption{IDBP-CNN (35.49 dB)}
    \vspace{1mm}
  \end{subfigure}  
  \caption{Deblurring of {\em Lena} image, Scenario 1. From top to bottom, fragments of: original image, blurred and noisy image, reconstruction of IRCNN, and reconstruction of the proposed IDBP-CNN.}
\label{deblurring_example2}
\end{figure}

\begin{figure}
\captionsetup[subfigure]{labelformat=empty}
  \centering
  \begin{subfigure}[b]{1\linewidth}
    \centering\includegraphics[width=150pt]{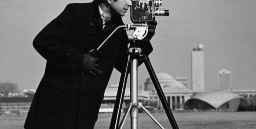}
    \centering\includegraphics[width=75.5pt]{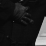}\\
    \caption{Original image}
    \vspace{1mm}
  \end{subfigure}
\\
  \begin{subfigure}[b]{1\linewidth}
    \centering\includegraphics[width=150pt]{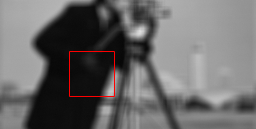}
    \centering\includegraphics[width=75.5pt]{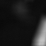}\\
    \caption{Blurred and noisy image}
    \vspace{1mm}
  \end{subfigure}
\\
  \begin{subfigure}[b]{1\linewidth}
    \centering\includegraphics[width=150pt]{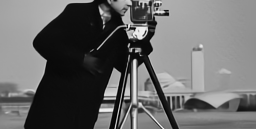}
    \centering\includegraphics[width=75.5pt]{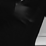}\\
    \caption{IRCNN (31.07 dB)}
    \vspace{1mm}
  \end{subfigure}
\\
  \begin{subfigure}[b]{1\linewidth}
    \centering\includegraphics[width=150pt]{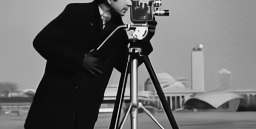}
    \centering\includegraphics[width=75.5pt]{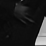}\\
    \caption{IDBP-CNN (31.32 dB)}
    \vspace{1mm}
  \end{subfigure}  
  \caption{Deblurring of {\em cameraman} image, Scenario 3. From top to bottom, fragments of: original image, blurred and noisy image, reconstruction of IRCNN, and reconstruction of the proposed IDBP-CNN.}
\label{deblurring_example3}
\end{figure}

\begin{table*}
\renewcommand{\arraystretch}{1.3}
\caption{Average deblurring results (PSNR in dB / SSIM) for for scenarios 1-4 on BSD68 dataset, and run-time (per image) on Intel i7-7500U CPU @ 2.70 GHz.} \label{table:bsd68_deblurring}
\centering
    \begin{tabular}{ | l | l | l | l | l | l | l |}
    \hline
            & Scenario 1 & Scenario 2 & Scenario 3 & Scenario 4 & Average & Time \\ \hline
    IDD-BM3D & 30.84 / 0.872 & 29.02 / 0.820 & 31.04 / {\bfseries 0.883} & 28.93 / 0.822 & 29.96 / 0.849 & 259s \\ \hline
    P\&P-BM3D & 30.41 / 0.865 & 28.53 / 0.806 & 30.78 / 0.880 & 28.61 / 0.814 & 29.58 / 0.841 & 85s \\ \hline
    IRCNN ($\sim$25 DNNs) & {\bfseries 31.17} / 0.877 & {\bfseries 29.31} / {\bfseries 0.832} & 30.84 / 0.865 & {\bfseries 29.16} / {\bfseries 0.830} & 30.12 / 0.851 & 34s \\ \hline
    IDBP-BM3D & 30.70 / 0.876 & 28.93 / 0.825 & 30.80 / {\bfseries 0.883} & 28.80 / 0.819 & 29.81 / 0.851 & 54s \\ \hline
    Auto-tuned IDBP-BM3D & 30.75 / 0.872 & 28.92 / 0.822 & 30.89 / 0.879 & 28.74 / 0.821 & 29.83 / 0.849 & 152s \\ \hline
    IDBP-CNN (1 DNN per scenario) & {\bfseries 31.17} / {\bfseries 0.882} & 29.19 / 0.830 & {\bfseries 31.12} / 0.878 & 29.13 / 0.828 & {\bfseries 30.15} / {\bfseries 0.855} & 35s \\ \hline
    Auto-tuned IDBP-CNN (1 DNN per scenario)  & 31.13 / 0.881 & 29.18 / 0.828 & 31.01 / 0.876 & 29.11 / 0.826 & 30.11 / 0.853 & 56s \\ \hline
    \end{tabular}
\end{table*}

\section{Conclusion}
\label{sec_conclusion}

In this work we introduced the Iterative Denoising and Backward Projections (IDBP) method for solving linear inverse problems using denoising algorithms. This method, in its general form, has only a single parameter that should be set according to a given condition.
We presented a mathematical analysis of this strategy and provided a practical way to tune its parameter. 
Therefore, it can be argued that our approach has less parameters that require tuning than the P\&P method. Specifically, for the noisy inpainting problem, the single parameter of the IDBP can be just set to zero, and for the deblurring problem our suggested automatic parameter tuning can be employed.
Experiments demonstrated that IDBP is competitive with state-of-the-art task-specific algorithms and with the P\&P approach for the inpainting and deblurring problems. 
It also achieves very promising results compared to IRCNN, while requiring significantly fewer denoising neural networks for solving an inverse problem, i.e. a single DNN instead of two dozen.

\appendices

\section{A Numerical Optimization Point of View on Problem (\ref{Eq_cost_func_our}) }
\label{app:insight}

A close look at problem (\ref{Eq_cost_func_our}) reveals that its $\tilde{\y}$ minimizer has a closed-form expression
\begin{align}
\label{Eq_ytilde_full_min}
\tilde{\y}^* = \H^\dagger \y + (\I_n - \H^\dagger \H ) \tilde{\x}.
\end{align}
Substituting (\ref{Eq_ytilde_full_min}) into (\ref{Eq_cost_func_our}), we have
\begin{align}
\label{Eq_cost_func_our_insight}
\minim{\tilde{\x}} \,\,\, \frac{1}{2(\sigma_e+\delta)^2} \| \H^\dagger \y - \H^\dagger \H \tilde{\x} \|_2^2 + s(\tilde{\x}).
\end{align}
While in the original problem (\ref{Eq_cost_func1}) the fidelity term measures the fitting of $\H\tilde{\x}$ to the measurements $\y = \H\x + \e$, in the new problem the fitting is done between $\P_H \tilde{\x}$ and $\H^\dagger \y = \P_H \x + \H^\dagger\e$, where $\P_H \triangleq \H^\dagger \H$ is the orthogonal projection onto the row space of $\H$. Assuming that $\H \in \Bbb R^{m \times n}$  ($m<n$) has full row rank, both operators $\H^T\H$ and $\P_H$ have rank $m$. However, though $\H^T\H$ may have very different eigenvalues, the eigenvalues of $\P_H$ can only be 1 in the row space of $\H$, and 0 in the null space of $\H$.

It is well known that many linear least squares optimization methods (e.g. conjugate gradients) perform better when the singular values of the linear operator are not spread over a wide range of values \cite{saad2003iterative}.
Therefore, if the prior $s(\tilde{\x})$ provides a strong restriction on $\Q_H \tilde{\x} \triangleq (\I_n - \P_H)\tilde{\x}$ given $\P_H \tilde{\x}$, then solving $c \| \H^\dagger \y - \P_H \tilde{\x} \|_2^2 + s(\tilde{\x})$ might be more stable than solving $c \| \y - \H \tilde{\x} \|_2^2 + s(\tilde{\x})$. Especially, assuming low noise level $\e \approx \0$, and recalling that good natural image priors are usually highly non-convex, a numerical optimization process (w.r.t. $\tilde{\x}$) for $c \| \P_H \x - \P_H \tilde{\x} \|_2^2 + s(\tilde{\x}) = c (\x-\tilde{\x})^T \P_H (\x-\tilde{\x}) + s(\tilde{\x})$ may end up with $\tilde{\x}$ closer to $\x$ than a numerical optimization process for $c \| \H \x - \H \tilde{\x} \|_2^2 + s(\tilde{\x}) = c (\x-\tilde{\x})^T \H^T\H (\x-\tilde{\x}) + s(\tilde{\x})$.

Despite having the above insight on the optimization problem (\ref{Eq_cost_func_our}), in order to get an efficient solver with a plug-and-play property for the prior $s(\x)$, we use alternating minimization for $\tilde{\x}$ and $\tilde{\y}$, instead of directly solving the problem for $\tilde{\y}$. 
We leave the rigorous study of the above numerical optimization direction for future research.

\section{Proof of Theorem \ref{theorem2}}
\label{app:thm}

We start with proving an auxiliary lemma.

\begin{lemma}
\label{lemma_aux}
Assuming that Condition \ref{cond1} holds, i.e. $\| \mathcal{D}(\z;\sigma) - \z \|_2 \leq \sigma B$ for any $\z$, we have
\begin{align}
\label{Eq_lemma}
\| \mathcal{D}(\z_1;\sigma) - \mathcal{D}(\z_2;\sigma) \|_2 \leq \| \z_1 - \z_2 \|_2 + 2\sigma B
\end{align}
for any $\z_1$ and $\z_2$ in $\Bbb R^n$.
\end{lemma}

\begin{proof}
Using the triangle inequality followed by Condition \ref{cond1}, we get the desired result
\begin{align}
\label{Eq_lemma_aux}
&\| \mathcal{D}(\z_1;\sigma) - \mathcal{D}(\z_2;\sigma) \|_2  \nonumber \\
& \,\,\,\, \leq \| \mathcal{D}(\z_1;\sigma) - \z_1 \|_2 + \| \mathcal{D}(\z_2;\sigma) - \z_2 \|_2 + \| \z_1 - \z_2 \|_2 \nonumber \\
& \,\,\,\, \leq \| \z_1 - \z_2 \|_2 + 2\sigma B.
\end{align}
\end{proof}

We now turn to the proof of the theorem.

\begin{proof}
By $\tilde{\x}_{k+1} = \mathcal{D}(\tilde{\y}_k;\sigma)$, and using the triangle inequality, we have
\begin{align}
\label{Eq_theorem2_main}
& \| \tilde{\x}_{k+1} - \x \|_2  \nonumber \\
& \,\,\,\, = \| \mathcal{D}(\tilde{\y}_k;\sigma) - \x \|_2 \nonumber \\
& \,\,\,\, \leq \| \mathcal{D}(\tilde{\y}_k;\sigma) - \mathcal{D}(\overline{\y};\sigma) \|_2 + \| \mathcal{D}(\overline{\y};\sigma) - \mathcal{D}(\x;\sigma) \|_2  \nonumber \\
& \,\,\,\,\,\,\,\,\,\,\,\, + \| \mathcal{D}(\x;\sigma) - \x \|_2 \nonumber \\
& \,\,\,\, \leq \| \tilde{\y}_k - \overline{\y} \|_2 + \| \overline{\y} - \x \|_2 + 5\sigma B \nonumber \\
& \,\,\,\, = \| \tilde{\y}_k - \overline{\y} \|_2 + \| \H^\dagger \e \|_2 + 5\sigma B,
\end{align}
where the second inequality uses Lemma \ref{lemma_aux} twice, and Condition \ref{cond1} for the last term, and the last equality follows from the fact that $\overline{\y} = \x + \H^\dagger \e$.
We turn to bound the first term in the right-hand side of (\ref{Eq_theorem2_main}). Because $\tilde{\y}_k = \H^\dagger\y + \Q_H \tilde{\x}_k$ and $\y = \H\x+\e$, we have
\begin{align}
\label{Eq_theorem2_main2}
& \| \tilde{\y}_k - \overline{\y} \|_2  \nonumber \\
& \,\,\,\, = \| (\H^\dagger\H\x + \H^\dagger \e + \Q_H \tilde{\x}_k) - (\x + \H^\dagger \e) \|_2 \nonumber \\
& \,\,\,\, = \| \Q_H (\tilde{\x}_k - \x) \|_2 \nonumber \\
& \,\,\,\, = \| \Q_H (\mathcal{D}(\tilde{\y}_{k-1};\sigma) - \x) \|_2 \nonumber \\
& \,\,\,\, \leq \| \Q_H (\mathcal{D}(\tilde{\y}_{k-1};\sigma) - \mathcal{D}(\overline{\y};\sigma)) \|_2  \nonumber \\
& \,\,\,\,\,\,\,\,\,\,\,\, + \| \Q_H (\mathcal{D}(\overline{\y};\sigma) - \mathcal{D}(\x;\sigma)) \|_2 +  \| \Q_H (\mathcal{D}(\x;\sigma) - \x) \|_2 \nonumber \\
& \,\,\,\, \leq K_\sigma \| \tilde{\y}_{k-1} - \overline{\y} \|_2 + K_\sigma \| \overline{\y} - \x \|_2 + \| \mathcal{D}(\x;\sigma) - \x \|_2 \nonumber \\
& \,\,\,\, \leq K_\sigma \| \tilde{\y}_{k-1} - \overline{\y} \|_2 + K_\sigma \| \H^\dagger \e \|_2 + \sigma B,
\end{align}
where the first inequality follows from the triangle inequality; the second inequality uses Condition \ref{cond3}, i.e. $\Q_H \mathcal{D}(\cdot;\sigma)$ is a contraction, for the first two terms, and $\|\Q_H\z\|_2 \leq \|\z\|_2$ for the last term; and the last inequality uses $\overline{\y} = \x + \H^\dagger \e$ and Condition \ref{cond1}.
Using recursion (recall that $K_\sigma<1$) we have
\begin{align}
\label{Eq_theorem2_main3}
\| \tilde{\y}_k - \overline{\y} \|_2 
& \leq K_\sigma^k \| \tilde{\y}_0 - \overline{\y} \|_2 + \frac{1-K_\sigma^k}{1-K_\sigma} ( K_\sigma \| \H^\dagger \e \|_2 + \sigma B )  \nonumber \\
& \leq K_\sigma^k \| \tilde{\y}_0 - \overline{\y} \|_2 + \frac{1}{1-K_\sigma} ( K_\sigma \| \H^\dagger \e \|_2 + \sigma B ).
\end{align}
Finally, substituting (\ref{Eq_theorem2_main3}) in (\ref{Eq_theorem2_main}) leads to (\ref{Eq_theorem2}).
\end{proof}

\section*{Acknowledgment}

The authors would like to thank Amir Beck for fruitful discussion, and the unknown reviewers for their important remarks that helped to improve the shape of the paper.
This work was supported by the European research council (ERC StG 757497 PI Giryes).

\bibliographystyle{ieeetr}

\bibliography{paper_bibliography}

\end{document}